\documentclass[11pt]{article}
\usepackage{fullpage}

\date{}

\usepackage[utf8]{inputenc} 
\usepackage[T1]{fontenc}    
\usepackage{hyperref}       
\usepackage{url}            
\usepackage{booktabs}       
\usepackage{amsfonts}       
\usepackage{nicefrac}       
\usepackage{microtype}      
\usepackage{amssymb}
\usepackage{amsmath}
\usepackage{amsthm}
\usepackage{thmtools}
\usepackage{thm-restate}
\usepackage{bbm}
\usepackage{microtype}
\usepackage{graphicx}
\usepackage{algorithm}
\usepackage{algorithmic}
\usepackage{hyperref}
\usepackage{color,xspace}
\usepackage{thm-restate}
\usepackage{subfigure}
\usepackage{cancel}

\allowdisplaybreaks

\newtheorem{assump}{Assumption}

\newcommand{\papertitle}{Optimal regret algorithm for \SBCOexp}
\newcommand{\SBCOexp}{Pseudo-1d Bandit Convex Optimization\xspace}
\newcommand{\SBCO}{\textsc{PBCO}\xspace}
\newcommand{\sbcalg}{\textsc{OptPBCO}\xspace}
\def \OGD{Online Gradient Descent for \SBCO\xspace}

\newcommand{\EXPBCO}{Kernelized Exponential Weights for \SBCO\xspace}

\newcommand{\f}{{\mathit{f}}}
\newcommand{\ex}{{\mathbb{E}}}
\newcommand{\loss}{{\mathit{\ell}}}
\newcommand{\pred}{{\mathit{g}}}
\newcommand{\regret}{{\mathcal{R}}}
\newcommand{\learner}{{\mathcal{A}}}
\newcommand{\predspace}{{\mathcal{G}}}
\newcommand{\actspace}{{\mathcal{W}}}
\newcommand{\E}{{\mathbb E}}
\newcommand{\N}{{\mathbb N}}
\newcommand{\R}{{\mathbb R}}

\newcommand{\cG}{{\mathcal G}}

\newcommand{\K}{{\textbf K}}
\newcommand{\bP}{{\textbf P}}

\newcommand{\bU}{{\textbf U}}

\newcommand{\cA}{{\mathcal A}}
\newcommand{\cD}{{\mathcal D}}
\newcommand{\cN}{{\mathcal N}}
\newcommand{\cX}{{\mathcal X}}

\newcommand{\cW}{{\mathcal W}}
\newcommand{\cB}{{\mathcal B}}
\newcommand{\cS}{{\mathcal S}}
\newcommand{\cR}{{\mathcal R}}
\newcommand{\cP}{{\mathcal P}}
\newcommand{\cQ}{{\mathcal Q}}
\newcommand{\tK}{\mathbf {K}}
\newcommand{\hf}{{\hat f}}

\newcommand{\w}{{\mathbf w}}
\newcommand{\tw}{{\tilde \w}}
\newcommand{\tf}{{\tilde f}}

\newcommand{\hl}{{\hat \ell}}

\newcommand{\bw}{{\bar \w}}

\newcommand{\x}{{\mathbf x}}

\newcommand{\z}{{\mathbf z}}
\newcommand{\p}{{\mathbf p}}
\newcommand{\q}{{\mathbf q}}

\newcommand{\bdelta}{{\boldsymbol \delta}}

\newcommand{\bsigma}{{\boldsymbol \sigma}}
\newcommand{\e}{\textbf{e}}
\newcommand{\1}{\mathbbm{1}}

\newtheorem{thm}{Theorem}
\newtheorem{lem}[thm]{Lemma}

\newtheorem{cor}[thm]{Corollary}
\newtheorem{defn}[thm]{Definition}

\newtheorem{rem}{Remark}

\newcommand{\red}[1]{\textcolor{black}{#1}}

\newcommand{\otilde}[1]{\tilde{O}\left(#1\right)}

\usepackage{natbib}
\usepackage{nameref}

\title{\bfseries \papertitle} 

\author{
Aadirupa Saha\thanks{Microsoft Research.} \and Nagarajan Natarajan\footnotemark[1] \and Praneeth Netrapalli\footnotemark[1] \and Prateek Jain\footnotemark[1] ~\thanks{The author is now at Google Research India. }
}
\begin{document}

%

%
\maketitle

\begin{abstract}
We study online learning with bandit feedback (i.e. learner has access to only zeroth-order oracle) where cost/reward functions $\f_t$ admit a "pseudo-1d" structure, i.e. $\f_t(\w) = \loss_t(\pred_t(\w))$ where the output of $\pred_t$ is one-dimensional. At each round, the learner observes context $\x_t$, plays prediction $\pred_t(\w_t; \x_t)$ (e.g. $\pred_t(\cdot)=\langle \x_t, \cdot\rangle$) for some $\w_t \in \mathbb{R}^d$ and observes loss $\loss_t(\pred_t(\w_t))$ where $\loss_t$ is a convex Lipschitz-continuous function. The goal is to minimize the standard regret metric. This pseudo-1d bandit convex optimization problem  (\SBCO) arises frequently in domains such as online decision-making or parameter-tuning in large systems. For this problem, we first show a lower bound of $\min(\sqrt{dT}, T^{3/4})$ for the regret of any algorithm, where $T$ is the number of rounds. We propose a new algorithm \sbcalg that combines randomized online  gradient descent with a kernelized exponential weights method to exploit the pseudo-1d structure effectively, guaranteeing the {\em optimal} regret bound mentioned above, up to additional logarithmic factors. In contrast, applying state-of-the-art online convex optimization methods leads to $\tilde{O}\left(\min\left(d^{9.5}\sqrt{T},\sqrt{d}T^{3/4}\right)\right)$ regret, that is significantly suboptimal in $d$.
\end{abstract}

\section{Introduction}
\label{sec:intro}

Online learning with bandit feedback is a cornerstone problem in the online learning literature and can be used to model a variety of practical systems where at each step $t$, the system takes an action $\w_t \in \mathbb{R}^d$ for which it incurs a loss of $\f_t(\w_t)$. Now, often times in practice, the action space has significantly more structure. For example, in large-scale parameter tuning the reward/loss is computed on a \textit{scalar} parameter predicted by an underlying ML model applied to the current context of system.  That is, the problem has a "pseudo-1d" structure in the loss functions $\f_t(\w) = \loss_t(\pred_t(\w; \x_t))$ where $g_t:\mathbb{R}^d \rightarrow \mathbb{R}$ is a one-dimensional function. 

We formulate this \SBCOexp (in Section~\ref{sec:prob}) as follows: given a data point, or context, $\x_t \in \cX$ at round $t$, the prediction of the learner is given by $\pred_t(\w_t; \x_t)$ for some $\w_t \in \actspace \subseteq \mathbb{R}^{d}$ and \emph{known} $\pred_t$, e.g. $\pred_t(\w_t; \x_t) =\langle \w_t, \x_t \rangle$. The learner then receives $\loss_t(\pred_t(\w_t; \x_t))$ from the adversary for some unknown convex, Lipschitz-continuous loss $\loss_t$. The goal is to minimize regret, i.e. the excess cumulative loss suffered by the learner over the best, fixed, parameter $\w^* \in \actspace$ in hindsight.  As mentioned above, the pseudo-1d structure arises naturally in online parameter tuning/decision making where the goal is to learn the optimal parameters $\w$ that govern the system, which can be very high-dimensional, but the dynamic reward $\loss_t$ depends only on a one-dimensional action $\pred_t$ taken by the system based on parameters $\w$ and the observed context $\x_t$.

The problem is a special case of the standard bandit convex optimization for which the state-of-the-art methods have regret of $O(d^{9.5}\sqrt{T})$ \citep{Bubeck17} or $O(\sqrt{d}T^{3/4})$ \citep{Flaxman+04}. So, the key question we answer in this paper is if and when the pseudo-1d structure can help obtain learning algorithms with better sample complexity or regret guarantees. For example, can we design an algorithm that has the optimal $\sqrt{T}$ regret in terms of $T$, but its regret is completely independent of $d$? Note that in the full-information setting, i.e., when full access to $\loss_t$ is available, the standard Online Gradient Descent (OGD) \citep{Zink03}  indeed obtains the optimal $\sqrt{T}$ regret independent of $d$. 

Somewhat surprisingly, our first result (Theorem~\ref{thm:lb} in Section~\ref{sec:lb}) shows that, even though the problem has a pseudo-1d structure, in the worst case, any algorithm will suffer a regret of $\min(\sqrt{dT}, T^{3/4})$ after $T$ rounds. That is, for large $T$, optimal regret has to scale with $d$. 

So, the next natural question is if we can design an algorithm to achieve the optimal regret. We answer that question in affirmative (Theorem~\ref{thm:ub} in Section~\ref{sec:algo}) by designing an efficient algorithm that indeed achieves the optimal regret when the loss function $\loss_t$ is convex and Lipschitz. Our method critically utilizes the pseudo-1d structure to define the algorithm in two regimes: a) for $d \geq \sqrt{T}$, we present a modification of the randomized gradient descent method by \citet{Flaxman+04} to get the rate optimal in this regime, b) for $d \leq \sqrt{T}$ we exploit a kernelized exponential weighting scheme similar to that of \citet{Bubeck17} to again obtain the optimal rate in this regime. \red{A key contribution of our work is that exploiting the problem structure also greatly simplifies the analysis and the proofs become significantly clearer (presented in Section~\ref{sec:algo}, Lemma~\ref{thm:ubkexp}), and much more palatable, than the general $d$-dimensional analysis by~\citet{Bubeck17}}. 

Now, it is instructive to compare our results against those of contextual bandit (CB) algorithms as the high level goal of both the formulations is similar. But, there are certain key distinctions between the two formulations. CB formulations work with general loss/reward functions while we restrict our methods to {\em convex Lipschitz}  functions only. On the other hand, CB methods are designed in general for discrete action and policy space (see Remark~\ref{rem:ccb} in Section~\ref{sec:prob}) unlike pseudo-1d bandit formulation that handles continuous prediction/action  space and infinite policy space. 

Finally, we present simulations in Section~\ref{sec:exp} that demonstrate the regret bounds on simple synthetic problems. Our contributions are summarized below:\\
1) A novel problem formulation that captures practical online learning scenarios with bandit feedback and structure in the reward/loss function.\\
2) A lower bound for the pseudo-1d bandit convex optimization problem -- in the worst case, any learning strategy suffers a regret of $O(\min(\sqrt{dT}, T^{3/4}))$.\\
3) A learning algorithm that is provably optimal, assuming the loss functions are convex and Lipschitz --- with a regret bound that matches the lower bound up to logarithmic factors. \\ \\ 
\textbf{Related Work.} \citet{Flaxman+04} initiated the study of bandit optimization for general convex functions and showed a regret guarantee of $O(\sqrt{d}T^{5/6})$ using online gradient-descent; with additional assumption of Lipschitzness, they improve the bound to $O(\sqrt{d}T^{3/4})$, and recently~\cite{Hazan16} and~\cite{Bubeck17} showed $\sqrt{T}$-regret (optimal in terms of $T$, but highly suboptimal in terms of $d$) using two different types of algorithms. Due to the fundamental nature of the problem, there is a long line of work in this space~\citep{Bubeck16multi,Chen18,Sahu18}, that look at certain types of losses (e.g. linear losses) \citep{Abernethy09,Yadkori+11}, different types of feedback (e.g. two-point feedback, as against one-point feeback in our work) \citep{Agarwal+11,Shamir17}, or different settings (stochastic vs adversarial) where improved regret bounds are possible \citep{Ghadimi,Shamir13,Mohri+16,sahatewari}. On the contrary, in the full information (online convex optimization) setting, where the gradient information of the loss function is known, \citet{Zink03} showed that online gradient descent achieves a regret of $O(\sqrt T)$ (which can be improved under additional assumptions~\citep{HazanAgKale07}). Contextual bandit learning has a vast literature and results focusing on finite/discrete action spaces (survey by~\citet{Bubeck12}). The state-of-the-art results for continuous action spaces (i.e. at each round, the learner receives context $\x_t$ and plays a value from $[0,1]$) is due to~\citet{krishnamurthy2019contextual,majzoubi2020efficient}; here, they work with a notion of ``smoothed'' regret, where each action is mapped to a smoothed action, and the learner also competes with a smoothed policy class (that maps context to action, akin to $\pred_t$). One key difference in the bandit learning literature is that typically there is no (or mild) assumption on the loss/reward function (See Remark~\ref{rem:ccb}).

\section{Problem Setup and Preliminaries}
\label{sec:prob}
The standard online (bandit) convex optimization framework proceeds in rounds: at round $t$, the learner \textit{plays} $\w_t \in \actspace \subseteq \mathbb{R}^d$ and receives the incurred loss $\f_t(\w_t)$ as feedback, for some convex $\f_t$ chosen adversarially. The ``action space'' $\actspace$ is restricted to be a closed convex set with diameter $W = \max_{\w, \w' \in \actspace} \|\w - \w'\|_2$. The goal of the (possibly randomized) learner $\learner$ is to have a bounded regret compared to a fixed $\w^* \in \mathcal{W}$ in hindsight that achieves the least cumulative loss, i.e. to minimize the regret defined as:
\begin{equation}
\regret_T(\mathcal{\learner}) = \sum_{t=1}^T \ex\big[f_t(\w_t)\big] - \sum_{t=1}^T \f_t(\w^*),
\label{eqn:regret}
\end{equation}
where $\w^* = \arg \min_{\w \in \mathcal{W}} \sum_{t=1}^T \f_t(\w)$, and $\E[.]$ is wrt to any randomness in $\learner$. In our formulation, at each round, the learner receives context $\x_t \in \cX$ , chooses parameters $\w_t$ and plays its prediction $\pred_t(\w_t; \x_t)$, and receives loss for this prediction; the loss functions chosen by the adversary at each round satisfies:
\begin{equation}
\f_t(\w_t) = \loss_t(\pred_t(\w_t; \x_t)),
\label{eqn:structuredloss}
\end{equation}  
for some $\pred_t: \actspace \times \cX \to \predspace \subseteq \mathbb{R}$, and bounded convex and $L$-Lipschitz $\loss_t: \predspace \to [0,C]$. 
Note that while the learner receives bandit feedback for $\loss_t$, it has complete knowledge of $\pred_t$, for example, $\pred_t(\cdot) = \langle \cdot, \x_t\rangle$. Thus, in particular, the learner has access to both zeroth- and first-order information for $\pred_t$ but only zeroth-order information for $\loss_t$. We refer to $\predspace$ as the prediction space. With this set up, we formally state the problem of interest below. 
\paragraph{\SBCOexp~(\SBCO):} Minimize \eqref{eqn:regret} where the functions $\f_t(.)$ admit the structure in~\eqref{eqn:structuredloss}, and $\loss_t, \x_t$ are chosen adversarially. 
\begin{rem}
\label{rem:1}
	Note that the goal is to minimize cumulative regret \eqref{eqn:regret} with respect to the best fixed $d$-dimensional parameter $\w^*$, though the learner plays in the prediction space $\predspace$ which is one-dimensional. 
\end{rem}

\begin{rem} [Applying bandit convex optimization]
\label{rem:bco} Ignoring the structure in~\eqref{eqn:structuredloss}, one can apply bandit convex optimization algorithms to \SBCO problem. The state-of-the-art result for online convex optimization with bandit feedback is by~\citet{Bubeck17}; using their algorithm gives a significantly sub-optimal regret bound of $O(d^{9.5}\sqrt{T})$. 
\end{rem}

\begin{rem} [Applying continuous contextual bandits] 
\label{rem:ccb} The recent work by~\citet{krishnamurthy2019contextual} provides optimal guarantees for contextual bandits with continuous actions (i.e. the learner plays an action from $[0,1]$ at each round). Applying their algorithm to our setting yields a ``smoothed'' regret (which is a weaker notion of regret, and not directly comparable to ours) of $O(T^{2/3}(Ld)^{1/3})$, where $L$ is Lipschitz constant of $\loss_t$. Note, however, that their guarantees apply to general losses and in particular do \emph{not} need convexity.
\end{rem}
In the (easier) setting of (bandit) stochastic convex optimization, there is a fixed unknown $\f(.)$ for which the learner obtains noisy evaluations. The goal is to minimize the expected value of the function, i.e., to bound: 
\begin{equation}
\bar{\regret}(\cA) := \min_{\w \in \actspace} \E_Z [\f(\w; Z) - \f(\w^*; Z)],
\label{eqn:bsco}
\end{equation}
where $\w^* = \arg \min_{\w \in \actspace} \E_Z [f(\w; Z)]$. Naturally, we can pose a stochastic version of the \SBCO problem where $\f$ admits the pseudo-1d structure.

\textbf{Notation.} Let $[n] = \{1,2, \ldots n\}$, for any $n \in \N$. For any $\delta> 0$, let $\cB_d(\delta)$ and $\cS_d(\delta)$ denote the ball and the surface of the sphere of radius $\delta$ in $d$ dimensions respectively.
Lower case bold letters denote vectors, upper case bold letters denote matrices.
$\bP_{\cX,\|\cdot\|}(\x)$ denotes the nearest point projection of a point $\x \in \R^d$ on to set $\cX \subseteq \R^d$ with respect to norm $\|\cdot\|$, i.e. $\bP_\cX(\x):= \arg \min_{\z \in \cX}\|\x - \z\|$. 
For any vector $\x \in \R^d$, $\|\x\|_2$ denotes the $\ell_2$ norm of vector $\x$. 
To be consistent with the literature, we will use $\f_t$ as a short-hand for $\loss_t(\pred_t(.))$ in this paper (as defined in \eqref{eqn:structuredloss}); and use $\pred_t(\w)$ as a short-hand for $\pred_t(\w; \x_t)$ when $\x_t$ is implicit from the context.

Below we give definitions that will be used in the remainder of the paper.

\textbf{(A1) Convexity:} 
For all $\w_1, \w_2 \in \cW$ and $\lambda \in [0,1]$, \\
\textbf{(i)} $\ell_t\big(\lambda g_t(\w_1) + (1-\lambda)g_t(\w_2)\big) \leq \lambda\ell_t(g_t(\w_1)) + (1-\lambda)\ell_t(g_t(\w_2))$ \\
\textbf{(ii)} $\f_t\big(\lambda \w_1 + (1-\lambda) \w_2\big) \leq \lambda\f_t(\w_1) + (1-\lambda)\f_t(\w_2)$.

\textbf{(A2) $L$-Lipschitzness:} 
For all $\w_1, \w_2 \in \cW$, $\|\ell_t(\pred_t(\w_1)) - \ell_t(\pred_t(\w_2))\|_2 \le L\|\pred_t(\w_1)-\pred_t(\w_2)\|_2$.

While we require the loss function to be convex, the learner can choose any bounded prediction function as stated below.

\textbf{(A3) Boundedness of $\pred_t$:} \textbf{(i)} $\pred_t \in \predspace =  [\alpha_\cW, \beta_\cW] \subseteq \mathbb{R}$, \textbf{(ii)} $\|\nabla_\w \pred_t(\w; \x)\| \leq D$, for all $\x \in \cX, \w \in \cW$. \red{Note A3(ii) implies $g_t$ is D-Lipschitz.}

\begin{rem}\label{rem:g}
Note that when $\pred_t$ is linear, i.e. $\pred_t(\w; \x) = \langle \w, \x \rangle$, then the above assumptions simplify: In particular, (a) \emph{\textbf{(A1) (i)}} $\iff$ \emph{\textbf{(A1) (ii)}}, (b) $\|\nabla_\w \pred_t(\w; \x)\| = \| \x \| \leq D$, where $D$ denotes the diameter of $\cX$, and $\pred_t \in [-DW, DW]$, where $W$ is the diameter of $\cW$.
\end{rem}

All detailed proofs are provided in the supplementary (Appendix A).


%

\section{A lower bound for $\SBCO$}
\label{sec:lb}
It does appear that the $\SBCO$ problem introduced in Section~\ref{sec:prob} is effectively a one-dimensional problem because the loss function $\loss_t$ is computed on a scalar. This raises the natural question as to when and if one can get rid of dimension dependence in the regret. Recall that existing bandit convex optimization techniques (Remark~\ref{rem:bco} in Section~\ref{sec:prob}) do suffer poly($d$) dependence. In the following we show that, in general, one cannot avoid the dependence on $d$, and in particular, we show a lower bound that is $\Omega(\sqrt{dT})$, in the regime $d = O(\sqrt{T})$. For larger $d$, any algorithm must suffer a regret that is $\Omega(T^{3/4})$.
\begin{thm}[Lower bound for \SBCO]
\label{thm:lb}
For any algorithm $\cA$ for the \SBCO~problem, there exists $\actspace \subseteq \mathcal{B}_d(1)$, and sequence of loss functions $f_1, \ldots f_T : \cW \mapsto \R$ where for any $t$, $\E[f_t(\cdot)] \in [0,1]$, the expected regret suffered by $\cA$ satisfies:
\begin{align*}
\E\big[ R_T(\cA)\big] & = \E\bigg[ \sum_{t=1}^T f_t(\w_t) - \min_{\w \in \cW} \sum_{t=1}^T f_t(\w) \bigg] \\
& \ge \frac{1}{32}\min\big(\sqrt{d T} , T^{3/4}\big).
\end{align*}
In particular, the lower bound holds under the assumptions \textbf{(A1)}, \textbf{(A2)} and \textbf{(A3)}.
\end{thm}
\paragraph{Proof Sketch.} We give a simple construction of problem instance to show the desired lower bound. We will work with linear model, i.e. $\pred_t(.) = \langle \x_t, \cdot \rangle$, and $\actspace = \frac{1}{\sqrt{d}}\{\pm 1\}^d$ which suffices for a lower bound. The idea is to divide the max rounds $[T]$ into $d$ equal length sub intervals (each of length $T/d$) $T_1,\ldots, T_{d}$ (let $T_0 = \emptyset$). Now, for $i \in [d]$, choose $\sigma_i \sim \text{Ber}(\pm 1)$, and set $\x_i = \e_i$. At round $t \in T_i = \Big\{\frac{T}{d}(i-1)+1, \ldots, \frac{T}{d}i \Big \}$, $i \in [d]$, adversary chooses $\x_t = \x_i$ and the loss function $\f_t(\w) = \mu\sigma_i(\w^\top\x_i) + \varepsilon_t, \text{ where } \varepsilon_t \sim \cN(0,\frac{1}{16}), \text{ for some constant } \mu > 0, \forall \w \in \cW$. For this problem instance, it is easy to show that $\w^* = -\frac{\bsigma}{\sqrt d} \in \cW$, where $\bsigma = (\sigma_1,\dots,\sigma_d)$. The learner's goal is then to figure out $\bsigma$. Now, we argue a lower bound for two regimes: \\
\textbf{Case $d \leq 16 \sqrt{T}$.} We can show that any learning strategy must suffer an expected regret of at least $\frac{\sqrt{dT}}{32}$ if we set $\mu = \frac{d}{16\sqrt{T}} < 1$ (used by the adversary for constructing $\f_t(\w)$ mentioned above). \\
\textbf{Case $d > 16\sqrt{T}$.} One can use an embedding trick, and simply ignore the $d - 16 \sqrt{T}$ dimensions. In this setup, we can argue that any learner must suffer a regret of at least $\frac{T^{3/4}}{32}$ by falling back on the first case. \\
Together, we get the desired lower bound. See Appendix A for details.
\begin{rem}
	Note that in the lower bound instance of Theorem~\ref{thm:lb}, $\loss_t$ and $\x_t$ are dependent random variables. In fact, this dependence is crucial for obtaining a lower bound that depends on the dimension $d$.
	It is indeed possible to design an algorithm that achieves $\otilde{\sqrt{T}}$ regret for the stochastic setting where $\loss_t$ is independent of $\x_t$. The main idea is this: all one needs to estimate is the minimizer of the one-dimensional function $\E\left[\loss_t\right]$. However, this situation does not seem to be of much interest and hence we do not provide a proof of this claim.
\end{rem}

\section{An optimal algorithm for $\SBCO$}
\label{sec:algo}
In this section, we develop a method for the $\SBCO$ problem in the adversarial setting, and show that it achieves a regret that matches the lower bound presented in Section~\ref{sec:lb}, up to logarithmic factors. The proposed solution operates in two regimes, mirroring the lower bound analysis: in one regime, when $d = O(\sqrt{T})$, it relies on a kernelized exponential weights scheme, and in the other regime, when $d$ is larger, it relies on an online gradient descent style algorithm. This method, called \sbcalg, is presented in Algorithm~\ref{algo:sbco}.
\begin{center}
\begin{algorithm}
   \caption{\sbcalg}
   \label{algo:sbco}
\begin{algorithmic}[1]
   \STATE {\bfseries Input:} 
   \STATE ~~~ max rounds $T$, dimensionality $d$
   \IF {$d \leq \frac{WLD\sqrt{T}}{C\log(L'T)}$}
   \STATE Run \EXPBCO~(Algorithm~\ref{algo:kexp}) with $\eta =  \frac{\sqrt{d\log(L'T)}}{C\sqrt{B T}}, T$
   \ELSE
   \STATE Run \OGD~(Algorithm~\ref{algo:ogd}) with $\eta = \frac{W\delta}{DC\sqrt{T}}, \delta = \Big( \frac{WDC}{3L \sqrt T} \Big)^{1/2}, \alpha = \delta,$ and $T$
   \ENDIF
\end{algorithmic}
\end{algorithm}
\end{center}
We now state our second key result of the paper --- \sbcalg achieves an optimal regret bound given below.
\begin{thm}[Regret bound for \sbcalg (Algorithm~\ref{algo:sbco})]
\label{thm:ub}
	If the loss functions $\loss_t: \predspace \to [0,C]$, $\f_t$ satisfy \emph{\textbf{(A1)}}, \emph{\textbf{(A2)}}, \emph{\textbf{(A3)}}, $\actspace = \cB_d(W)$, the expected regret of the \SBCO~learner presented in Algorithm~\ref{algo:sbco} can be bounded as:
	\begin{align*}
	\E[& \regret_T(\cA_{\sbcalg})] \le  \\
	& 2\sqrt 2 \min \bigg(C\sqrt{dT\log(L'T)}, \sqrt{WLDC}T^{3/4}\bigg)	
	\end{align*}
	where $L' = LDW$ and the expectation $\E[\cdot]$ is with respect to the algorithm's randomization.
\end{thm}
\begin{proof} The bound follows from Lemmas~\ref{thm:ubkexp} and~\ref{thm:ubogd}, the choice of parameters given in steps 4 and 6 of Algorithm~\ref{algo:sbco}, and noticing that when $d$ is larger than the threshold in step 3 of the Algorithm, OGD (Algorithm~\ref{algo:ogd}) achieves a smaller regret than Kernelized Exponential Weights (Algorithm~\ref{algo:kexp}).
\end{proof}
\begin{cor}
When $\pred_t$ is linear, i.e. $\pred_t(.) = \langle \x_t, \cdot \rangle$, then $D$ is the diameter of $\cX$.
\end{cor}
A few remarks are in order.
\begin{rem}
\sbcalg requires the knowledge of the Lipschitz constant $L$ (e.g. in Step 3) of unknown loss $\loss_t$. This is a standard assumption made in the bandit convex optimization literature~\citep{Flaxman+04}.
\end{rem}
\begin{rem}\label{rem:stochastic}
It is straight-forward to state a result similar to Theorem~\ref{thm:ub} for the stochastic version of the \SBCO problem. 
\end{rem}
\subsection{\red{Regime $d = \widetilde{O}(\sqrt{T})$: Kernelized Exp. Weights}}
The key idea in our approach is to use a kernelized exponential weights scheme that exploits the pseudo-1d structure in the loss function. Exponential weights is a popular online learning algorithm for contextual bandits. Recently~\cite{Bubeck17} developed a meticulous kernel method that uses exponential weight update at its core to prove $O(\sqrt{T})$ regret for general convex (and Lipschitz) functions. Their approach hinges on using a smoothing operator (kernel) to obtain an estimator of the loss function $\f_t$ (the analogous estimator is fairly straight-forward in the multi-arm bandit setting) in the bandit convex optimization setting. 

\red{In the general $d$-dimensional setting, defining a kernel such that the resulting estimator of $\f_t$ is both (almost) unbiased and has bounded variance turns out to be extremely complicated and incurs large polynomial factors in dimension $d$. But, we can exploit the pseudo-1d structure in our setting to define a relatively simple kernel in the \textit{one-dimensional} prediction space instead. A key benefit of using the simple 1-d kernel is that much of the analysis in~\cite{Bubeck17} can be greatly simplified, and the proofs become significantly easier to follow.} 

\red{Before describing the main ideas of the algorithm, we need some notation and definitions set up.} Let $\p_t$ denote the distribution over parameters $\cW$ maintained by the learner at round $t$. Also let $\cG_t:= \{ \pred_t(\w,\x_t) \mid \w \in \cW \} \subseteq \R$, for any $t \in [T]$, and $\cW_t(y):= \{\w \in \cW \mid \pred_t(\w,\x_t) = y\}$, for $y \in \cG_t$. 
Given this, we obtain a one dimensional distribution $\q_t \in \cQ_t$ over $\cG_t$ from $\p_t$ as follows: $d \q_t(y):= \int_{\cW_t(y)} d \p_t(\w)$, $\forall y \in \cG_t$. 

\red{The kernelized exponential weights scheme crucially uses a kernel map to obtain a smooth estimate of the loss function on the action space based on a single point evaluation. The key observation we make is that, in our setting, it suffices to define such a kernel over the \textit{scalar} \textit{prediction} space than over the $d$-dimensional action space as in~\cite{Bubeck17}.} This $1$-dimensional kernel map, denoted $\K_t'$, is carefully constructed at each round $t$ based on $\q_t$ and the observed context $\x_t$ as given below:

\begin{defn}
	\label{def:kernel}
Given a distribution $\q_t$ over $\cG_t$, and $\epsilon > 0$, we define a one-dimension kernel $\K'_t: \cG_t \times \cG_t \mapsto \R_+$ as: 
\begin{align*}
\emph{\K}'_t(y,y') =  
\begin{cases}
\dfrac{\1\big( y \in [y', \bar y] \big)}{|y' - \bar y|},
~ \text{ if } |y' - \bar y| \ge \epsilon,\\
\dfrac{\1\big( y \in [\bar y - \epsilon, \bar y] \big)}{\epsilon},
~ \text{ when } y' \in [\bar y-\epsilon,\bar y + \epsilon]\\
\end{cases},
\end{align*}
where $\bar y:= \E_{\q_t}[y]$.
\end{defn}

For the kernel $\K_t'$ defined above, we can verify that $\int_{\predspace_t}\K_t'(y,y')d y = 1$ for every $y' \in \predspace_t$. 
Further we define a linear operator on any $\q \in \cQ_t$ (a smoothing of $\q$ w.r.t. $\K_t'$) as:
\begin{equation}
	\label{eq:ktprime}
\K_t'\q(y) := \int_{y' \in \predspace_t}\K_t'(y,y')d\q(y') ~~~ \forall y \in \predspace_t.
\end{equation}

This operator is particularly useful because for any valid probability measure $\q \in \cQ_t$, the map $\K_t'\q$ also defines a valid probability distribution over $\predspace_t$ (a precise statement is proved in Lem. \ref{lem:validp}, Appendix A.2). 

\textbf{Algorithm (main ideas). } We start with maintaining uniform weight over the $\actspace$: $\p_1 \leftarrow \frac{1}{\text{vol}(\cW)}$. At any time $t \in [T]$, upon receiving $\x_t$, we first compute the effective scalar decision space $\cG_t$ and sample a $y_t \in \cG_t$ according to the smoothed distribution of $\K_t'\q_t$. 
However, since the task is to choose a prediction point from the $d$-dimensional space $\cW$, we pick any (uniformly) random $\w_t$ that maps to $y_t$, i.e. $\w_t \in \cW_t(y_t)$ uniformly at random (Line 7 in Algorithm~\ref{algo:kexp}). 
Upon receiving the zeroth-order feedback $f_t(\w_t)$, we estimate the loss at each point $\w \in \cW$ as follows: 
\[ 
\tf_t(\w) \leftarrow \dfrac{\f_t(\w_t)}{\tK_t'\q_t(y_t)}\K_t'(y_t,y), ~~ \forall ~\w \in \cW.
\]
Note the above loss estimate $\tilde f_t$ ensures for a fixed $y \in \cG_t$, $\tilde f_t(\w)$ is same for all $\w \in \cW_t(y)$ (as justified by the structure: $f_t(\cdot) = \ell_t(g_t(\cdot))$). Finally, using the (estimated) loss $\tilde f: \cW \mapsto \R$, we update $\p_t$ identical to the standard exponential weights algorithm:
\[ 
\p_{t+1}(\w) \leftarrow \dfrac{\p_t(\w)\exp\big( -\eta \tilde \f_t(\w) \big)}{\int_\tw \p_t(\tw)\exp\big( -\eta \tilde \f_t(\tw) \big)d\tw }, ~~ \forall ~\w \in \cW.
\]
Algorithm~\ref{algo:kexp} summarizes the proposed kernelized exponential weights scheme for~\SBCO.

\begin{center}
	\begin{algorithm}
		\caption{\EXPBCO ~ }
		\label{algo:kexp}
		\begin{algorithmic}[1]
			\STATE {\bfseries Input:} learning rate: $\eta >0$, $\epsilon > 0$, max rounds $T$.
			\STATE {\bfseries Initialize:}  $\w_1 \leftarrow \boldsymbol 0, \p_1 \leftarrow \frac{1}{\text{vol}(\cW)}$. 
			\FOR {$t = 1,2, \cdots T$}
			\STATE Receive $\x_t$, and define $\cG_t:= \{ \pred_t(\w,\x_t) \mid \w \in \cW \} \subseteq \R$
			\STATE Define $\q_t$ such that $d \q_t(y):= \int_{\cW_t(y)} d \p_t(\w)$, $\forall y \in \cG_t$, where $\cW_t(y):= \{\w \in \cW \mid g_t(\w,\x_t)=y\}$ 
			\STATE Using $\x_t$ and $\q_t$, and given $\epsilon$, define kernel $\K_t': \cG_t \times \cG_t \mapsto \R$ (according to Definition \ref{def:kernel})
			\STATE Sample $y_t \sim \tK_t'\q_t$ and pick any $\w_t \in \cW_t(y_t)$ uniformly at random
			\STATE Play $\pred_t(\w_t; \x_t)$ and receive loss $\f_t(\w_t) = \loss_t(\pred_t(\w_t; \x_t))$ 
			\STATE $\tf_t(\w) \leftarrow \dfrac{\f_t(\w_t)}{\tK_t'\q_t(y_t)}\K_t'(y_t,y)$, for all $\w \in \cW(y), \forall y \in \cG_t$ \hspace{1cm}  $\triangleright$ {\color{cyan} estimator of $f_t$}
			\STATE $\p_{t+1}(\w) \leftarrow \dfrac{\p_t(\w)\exp\big( -\eta \tilde \f_t(\w) \big)}{\int_\tw \p_t(\tw)\exp\big( -\eta \tilde \f_t(\tw) \big)d\tw }$, for all $\w \in \cW$ 
			\ENDFOR
		\end{algorithmic}
	\end{algorithm}
\end{center}

We show in the following Lemma that the regret bound for Algorithm~\ref{algo:kexp} is bounded by $\tilde{O}(\sqrt{dT})$. Exploiting the problem structure gets us significantly improved dependence on $d$ compared to the original result by~\citet{Bubeck17} for the general case (as stated in Remark~\ref{rem:bco}).

\begin{lem}[Regret bound for Algorithm \ref{algo:kexp}]
	\label{thm:ubkexp}
	If the losses $\loss_t: \predspace \to [0, C]$ and $\pred_t$, $t \in [T]$ satisfy \emph{\textbf{(A1) (i)}}, \emph{\textbf{(A2)}}, and  \emph{\textbf{(A3)}}, then for the choice of $\{\K_t\}_{t \in [T]}$ as defined in Definition \ref{def:kernel}, the expected regret of Algorithm \ref{algo:kexp}, with learning rate $\eta = \Big(\frac{2d\log(L'T)}{B C^2T}\Big)^{\frac{1}{2}}$ and $\epsilon = \frac{1}{3LT}$, can be bounded as:
	\begin{align*}
	\E[\regret_T] & \le 4 + 2\sqrt 2\Bigg(\sqrt{d B C^2T\log(L ' T)} \Bigg) \\
	& = O\Big( C\sqrt{dT \log (L'T)} \Big)
	\end{align*}
	where $B = 2\Big(1 + \ln (3LT)  + \ln\Big(\beta_\cW - \alpha_\cW\Big)\Big)$, $L' = L D W$, $W = \text{Diam}(\cW)$ and the expectation $\E[\cdot]$ is with respect to the algorithm's randomization.
\end{lem}
\paragraph{Proof sketch.} Detailed proof (and supporting lemmas) is presented in Appendix A. Here, we sketch all its key constituents. The proof relies on key properties of the aforementioned 1-d kernel map, shown in Lemma~\ref{lem:krnl_prop}.
We start by analyzing the expected regret w.r.t. the optimal point $\w^* \in \cW$ (denote $y_t^* = \pred_t(\w^*)$ for all $t \in [T]$). Define $\forall y \in \cG_t, ~\tilde \ell_t(y) := \tf_t(\w)$, for any $\w \in \cW(y)$.
Also let $\mathcal H_t = {\boldsymbol \sigma}\big(\{\x_\tau,\p_\tau,\w_\tau,f_\tau\}_{\tau = 1}^{t-1} \cup \{\x_t,\p_t\}\big)$ denote the sigma algebra generated by the history till time $t$. 
Then the expected cumulative regret of Algorithm \ref{algo:kexp} over $T$ time steps can be bounded as:
\begin{align}
\label{eq:reg_kp_main}
\nonumber &\E[R_T(\w^*)] := \E\bigg[\sum_{t = 1}^T\Big(\ell_t(g_t(\w_t; \x_t)) - \ell_t(g_t(\w^*; \x_t))\Big)\bigg]\\
\nonumber &= \E\bigg[\sum_{t = 1}^T\Big(\ell_t(y_t) - \ell_t(y_t^*)\Big)\bigg] = \E\bigg[\sum_{t=1}^{T}\big<\K'_t\q_t - \bdelta_{y^*_t}, \loss_t \big>\bigg] \\ 
\nonumber & \le 6\epsilon L T + 2\sum_{t=1}^{T}\E\bigg[\big< \K'_t(\q_t - \bdelta_{y^*_t}), \loss_t \big>\bigg] \\
%
\nonumber & \overset{(a)}{=} 6\epsilon L T  + 2\sum_{t=1}^{T}\E\bigg[\sum_{t=1}^{T}\E_{y_t \sim \K'_t\q_t}\Big[\big< \q_t - \bdelta_{y^*_t}, \tilde \loss_t \big> \mid \mathcal H_t \Big]\bigg] \\
& = 6\epsilon L T + 2\sum_{t=1}^{T}\E\bigg[\sum_{t=1}^{T}\E_{y_t \sim \K'_t\q_t}\Big[\big< \p_t - \bdelta_{\w^*}, \tilde f_t \big> \mid \mathcal H_t  \Big]\bigg]
\end{align}
where the last equality follows by Lemma~\ref{lem:p_eqv}, and by $\big< \bdelta_{\w^*}, \tf_t \big> = \tf_t(\w^*) = \tilde \loss_t(y^*_t) = \big< \bdelta_{y^*_t}, \tilde \loss_t \big>$; $(a)$ and the first inequality rely on the properties of the kernel in Lemma~\ref{lem:krnl_prop}.
Let us denote by $\p^*$ a uniform measure on the set $\cW_\kappa:=\{\w \mid \w = (1-\kappa)\w^* + \kappa \w', \text{ for any } \w' \in \cW\}$ for some $\kappa \in (0,1)$. 
We can then show that the inner expectation in~\eqref{eq:reg_kp_main} can be bounded by 
$\sum_{t=1}^{T}\E_{y_t \sim \K'_t\q_t}[\big< \p_t ,\tilde f_t \big> - \big<\p^*,\tilde f_t\big>] + \kappa L D W T$
using the assumption that $g_t$ is $D$-Lipschitz, and a certain adjoint operator on the kernel map is $L$-Lipschitz. The term $\sum_{t=1}^{T}\big< \p_t - \p^*, \tf_t \big>$ can be bounded (via Lemma \ref{lem:elem}) by $\dfrac{KL(\p^*||\p_1)}{\eta} + \frac{\eta}{2}\big< \q_t, \tilde \loss_t^2\big>$. Now, the second term $\E_{y_t \sim \K'_t\q_t}\Big[\big< \q_t, \tilde \loss_t^2\big>\Big]$ relates to the variance of the loss estimator, and can be bounded by a constant, ensured by our choice of the 1d-kernel; and the first, KL divergence, term can be bounded by $d \log \frac{1}{\kappa}$ by the definition of $\p^*$. Plugging these bounds in \eqref{eq:reg_kp_main}, letting $L' = L D W$, and setting $\kappa = \frac{1}{L'T}$, $\epsilon = \frac{1}{3L T}$,~\eqref{eq:reg_kp_main} yields:
\begin{align*}
\E[R_T(\w^*)] &\leq& O(1) + 2\Bigg(\frac{d\log L' T}{\eta}  + \frac{\eta B C^2T}{2} \Bigg)
\end{align*}
By choosing $\eta$ to minimize the RHS above, the proof is complete. 

We observe from Lemma~\ref{thm:ubkexp} that when $d$ is small and constant, the bound behaves like $\sqrt{T}$ but when $d$ is large, say, $d = T^{2/3}$, the bound behaves like $T^{5/6}$. In what follows, we show that an online gradient descent style algorithm achieves a regret that scales as $T^{3/4}$ \emph{independent of} $d$. 

\subsection{\red{Larger $d$: Online Gradient Descent}} Consider the standard online gradient descent algorithm of~\cite{Zink03}, but with an estimator in lieu of the true gradient as in~\cite{Flaxman+04} to deal with bandit feedback. The key observation here is that we can perform the gradient estimation much more accurately exploiting the pseudo-1d structure. In particular, using the chain rule, one can write the gradient of the loss function wrt to $\w$ as:
\begin{equation}  
\nabla_\w \f_t(\w) = \nabla_\w \loss_t(\pred_t(\w; \x_t)) = \loss_t' (\pred_t(\w; \x_t))\nabla_\w \pred_t(\w; \x_t)
\end{equation}
Notice that because we have access to $\pred_t$, we know the $d$-dimensional gradient part accurately. The only unknown part in the equation above is the \emph{scalar} quantity which is $\loss_t' (\pred(\w_t; \x_t))$. For this, we can use the one-point estimator as in~\cite{Flaxman+04}, which in expectation gives the gradient wrt to not the actual loss $\loss_t$ but wrt to a smoothed loss, as stated in the following lemma.
\begin{lem}
\label{lem:gradest}
Fix $\delta > 0$ and let $u$ take 1 or -1 with equal probability. Define the one-point gradient estimator, $\hat{\nabla}\loss_t (a) := \frac{1}{\delta} \loss_t \big( a + \delta u\big)u$. Then:
\begin{align*}
\nabla_\w  \E_{u}\big[ \loss_t\big(\pred_t(\w_t; \x_t) + \delta u\big) \big] &=& \E_{u}\big[\hat{\nabla}\loss_t\big(\pred_t(\w_t; \x_t)\big)\big] \\ & &.\ \nabla_\w \pred_t(\w_t; \x_t)
\end{align*}
\end{lem}
The resulting online gradient descent method for \SBCO is given in Algorithm~\ref{algo:ogd}. In Lemma~\ref{thm:ubogd}, we give the $O(T^{3/4})$ regret bound for the algorithm.
\begin{center}
\begin{algorithm}
   \caption{\OGD}
   \label{algo:ogd}
\begin{algorithmic}[1]
   \STATE {\bfseries Input:} 
   \STATE ~~~ Perturbation parameter: $\delta>0$, $\alpha \in (0,1]$, learning rate: $\eta>0$, max rounds $T$
   \STATE {\bfseries Initialize:}
   \STATE ~~~ $ \w_1 \leftarrow 0$ 
   \FOR {$t = 1,2, \cdots T$}
   \STATE Sample $u \sim \bU\big(\cS_1(1)\big)$ (i.e. select $u$ uniformly from $\{-1,1\}$)
   \STATE Receive $\x_t$
   \STATE Project $\w_{t} \leftarrow \bP_{\actspace_{\alpha}}(\w_{t})$, where $\cW_{\alpha} = \{\w \in \actspace \mid \pred_t(\w; \x_t) \in \predspace - \alpha \}$
   \STATE Play $a_t = \pred_t(\w_t; \x_t) + \delta u$ and receive loss $\loss_t(a_t)$
   \STATE Update $\w_{t+1} \leftarrow \w_{t} - \eta\Big[ \frac{1}{\delta}\loss_t(a_t)u\nabla\pred_t(\w_t; \x_t)\Big]$ \hfill $\triangleright$ {\color{cyan} One-point estimator of $\nabla \f_t(\w_t)$}
   \ENDFOR
\end{algorithmic}
\end{algorithm}
\end{center}
 
\begin{lem}[Regret bound for Algorithm \ref{algo:ogd}]
\label{thm:ubogd}
Consider $\cW = \cB_d(W)$. If the losses $\f_t: \cW \to [0, C]$ and $\pred_t$, $t \in [T]$ satisfy \emph{\textbf{(A1) (ii)}}, \emph{\textbf{(A2)}}, and \emph{\textbf{(A3) (ii)}}, then setting $\eta = \frac{W\delta}{DC\sqrt{T}}$, $\delta = \Big( \frac{WDC}{3L \sqrt T} \Big)^{1/2}$, and $\alpha = \delta$,  
the expected regret of Algorithm \ref{algo:ogd} can be bounded as:
\[
\E[\regret_T(\learner)] \leq 2\sqrt{3WLDC} T^{3/4},
\]
where the expectation $\E[\cdot]$ is with respect to the algorithm's randomization. 
\end{lem}

Thus, we are able to guarantee optimal regret bound for \sbcalg matching the lower bound, by falling back on a suitably modified OGD algorithm when $d$ is sufficienly large. 

\begin{rem} [Assumptions for OGD vs Kernelized Exponential Weights] To show the regret bound for Algorithm~\ref{algo:kexp}, we only need convexity of the one-dimensional function $\loss_t$ unlike in the OGD case (Algorithm~\ref{algo:ogd}) where we need convexity of $\f_t$ in the $d$-dimensional parameter $\w$. In particular, our analysis of kernelized exponential weights method (in Lemma~\ref{thm:ubkexp}) does not need other assumptions on $\pred_t$ other than boundedness, which may be counter-intuitive (for example, consider when $\pred_t$ is possibly non-convex and $\loss_t$ is the identity function). But note that the analysis relies on the complete knowledge of $\pred_t$ and ignores the computational complexity. To be able to implement Algorithm~\ref{algo:kexp} efficiently, we will need some nice property of $\pred_t$ like convexity. 
\end{rem}

The following remark shows that pseudo-1d structure helps improve known bounds for bandit convex optimization by a factor of $\sqrt{d}$ at least.
\begin{rem}\label{rem:linband}
Consider the simple setting of bandit convex optimization when the loss functions are linear, $\f_t(\w_t) = \langle \w_t, \xi_t \rangle$, where $\xi_t$ is the cost vector chosen by the adversary, not revealed to the learner. It is known that, for bandit linear optimization, the minimax optimal regret is $\Theta(d\sqrt{T})$~\citep{Shamir15}. Note that, in contrast, the context vector $\x_t$ is revealed to the learner in our setting, and only the (scalar) loss computed on the linear model $\langle \w_t, \x_t\rangle$ is not revealed, which captures typical online decision making setting. This way of posing the problem helps us leverage the structure, and get a better dependence on $d$.
\end{rem}

\section{Simulations}
\label{sec:exp}
We present synthetic experiments that showcase the regret bounds established in Section~\ref{sec:algo}. We work with a linear $\pred_t$ for all the experiments. We fix $\cW = \cB_d(1)$, context vectors from $\{\|\x_t\|_2 \le 1\}$, and the two loss functions (a) $\f_t(\w) = (\langle \w, \x_t \rangle - y^*_t)^2$ where $y^*_t = \langle \w^*, \x_t\rangle$, for a fixed $\w^* \in \cB_d(1)$, and (b) $\f_t(\w) = |\langle \w, \x_t \rangle - y^*_t|$. 
 The details on implementing Algorithm~\ref{algo:kexp} are given in Appendix B.
\paragraph{OGD vs Kernelized Exponential Weights for $\SBCO$.}
In Figure~\ref{fig:results} (a)-(b), we show the expected regret of Algorithm~\ref{algo:ogd} on the synthetic problem (averaged over 50 problem instances), scaled by $1/t^{3/4}$ at round $t$, for the two loss functions; this, according to Lemma~\ref{thm:ubogd}, ensures that the expected regret converges to a numerical constant, independent of $d$, with increasing rounds. We observe this is indeed the case for different $d$ values. In Figure~\ref{fig:results} (c)-(d), we show the expected regret of Algorithm~\ref{algo:kexp} on this problem (averaged over 50 problem instances), scaled by $1/\sqrt{t}$ at round $t$, for the two loss functions; this, according to Lemma~\ref{thm:ubkexp}, ensures that the regret converges to $O(\sqrt{d})$, with increasing rounds; notice that, e.g., in (c), for different $d$ values, the converged scaled regret is $\gamma \sqrt{d}$ where $\gamma \approx 0.02/\sqrt{40} \approx 0.015/\sqrt{20} \approx 0.01/\sqrt{10} \approx 0.003$.

\begin{figure*}
\centering
\subfigure[Alg~\ref{algo:ogd} (squared $\loss$)]{\includegraphics[trim={0cm 0cm 0cm 0cm},clip,scale=0.45]{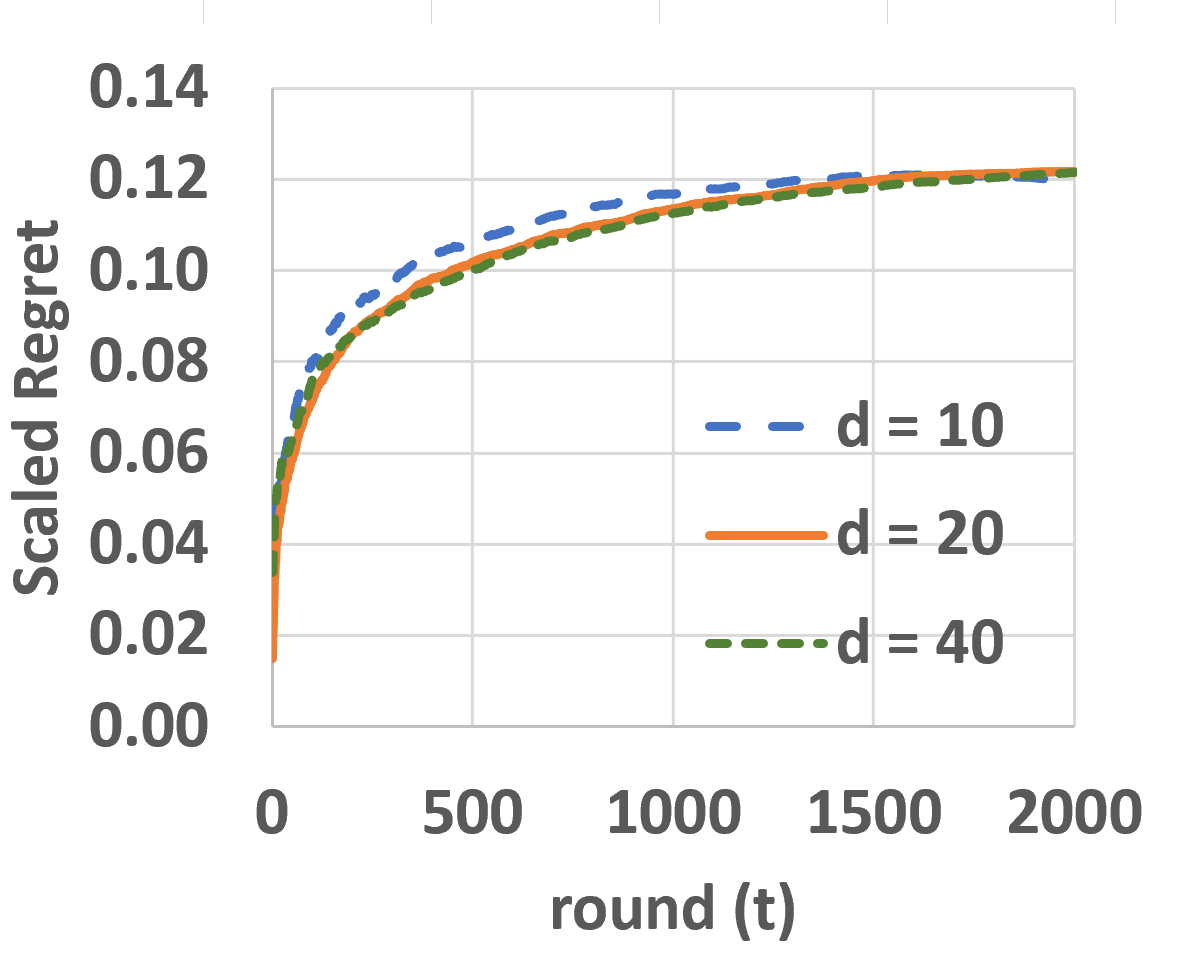}}
\subfigure[Alg~\ref{algo:ogd}  (abs. $\loss$)]{\includegraphics[trim={0cm 0cm 0cm 0cm},clip,scale=0.45]{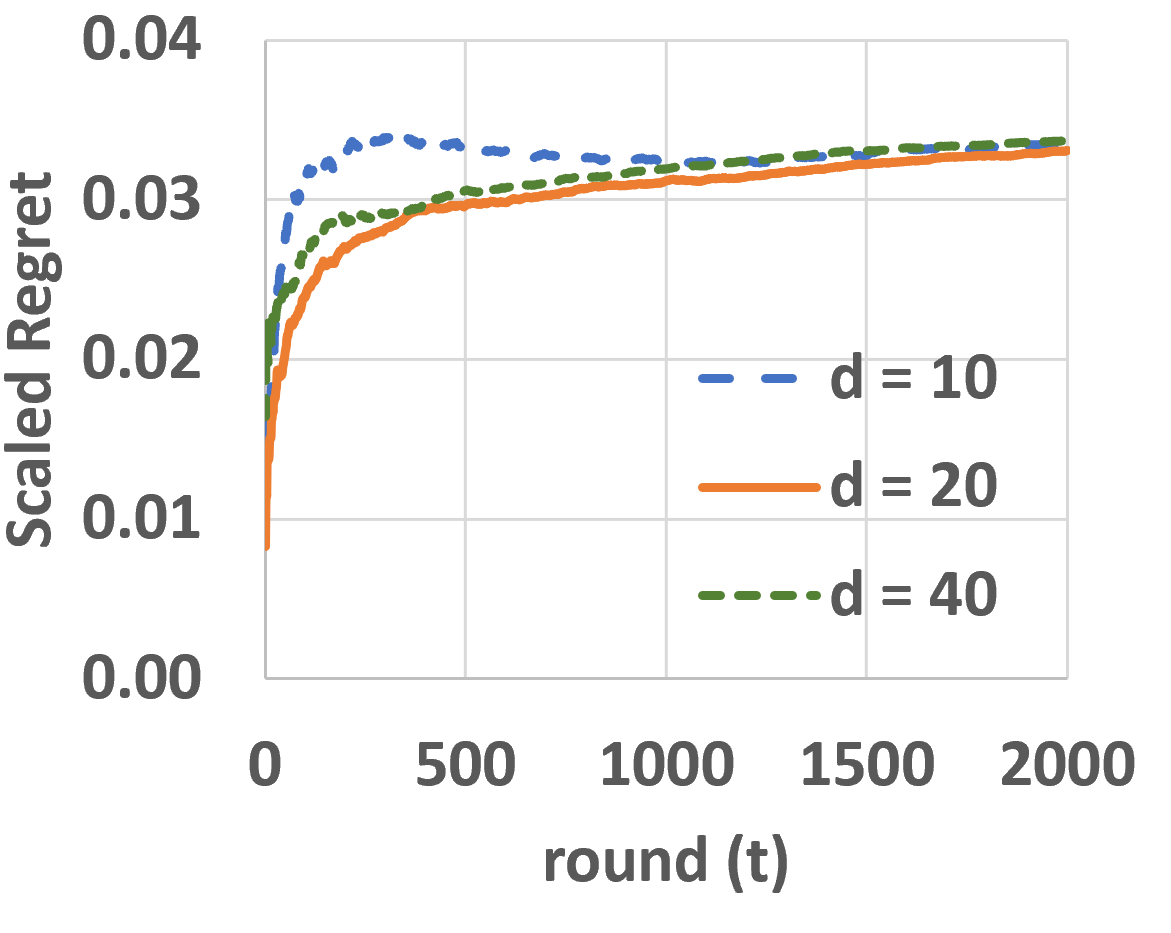}}
\subfigure[Alg~\ref{algo:kexp} (squared $\loss$)]{\includegraphics[trim={0cm 0cm 0cm 0cm},clip,scale=0.45]{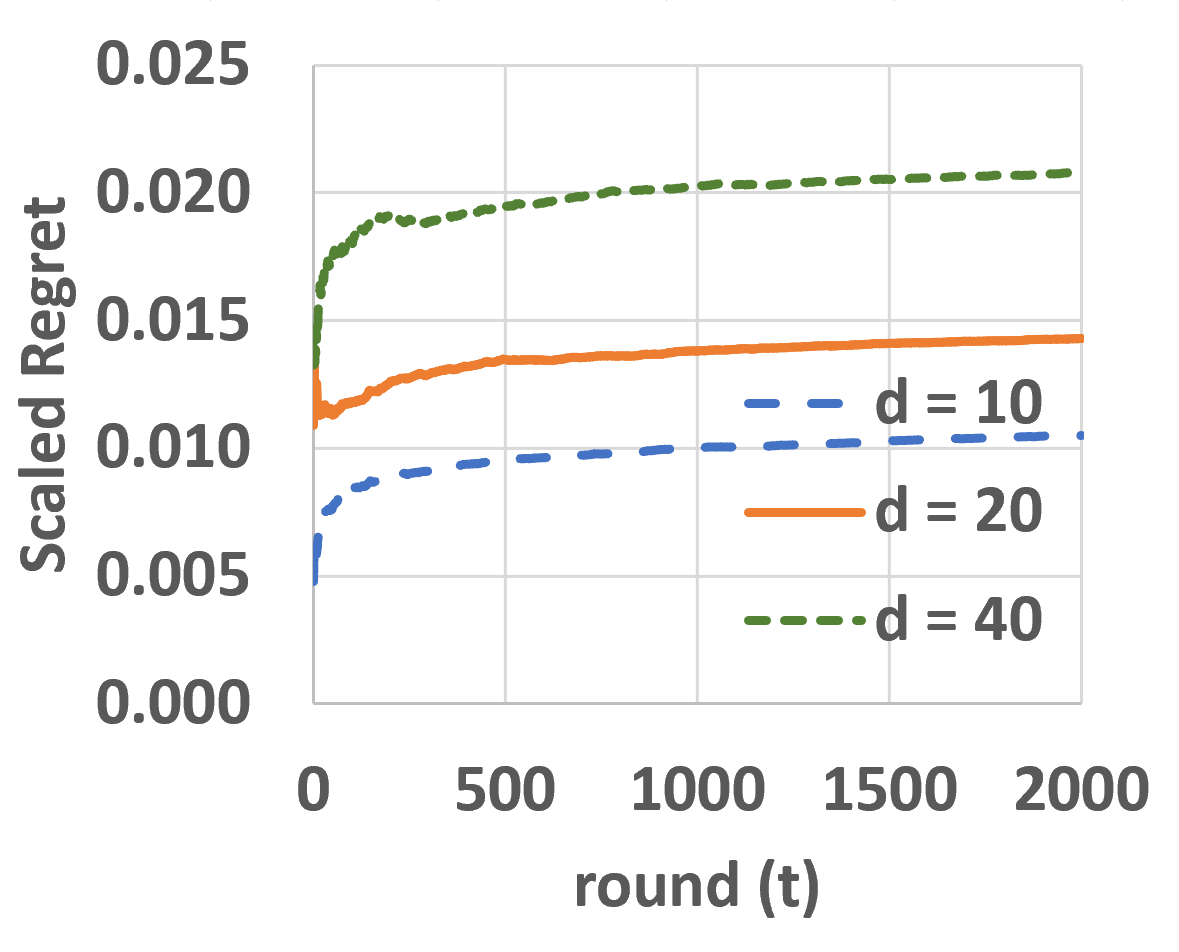}}
\subfigure[Alg~\ref{algo:kexp} (abs. $\loss$)]{\includegraphics[trim={0cm 0cm 0cm 0cm},clip,scale=0.45]{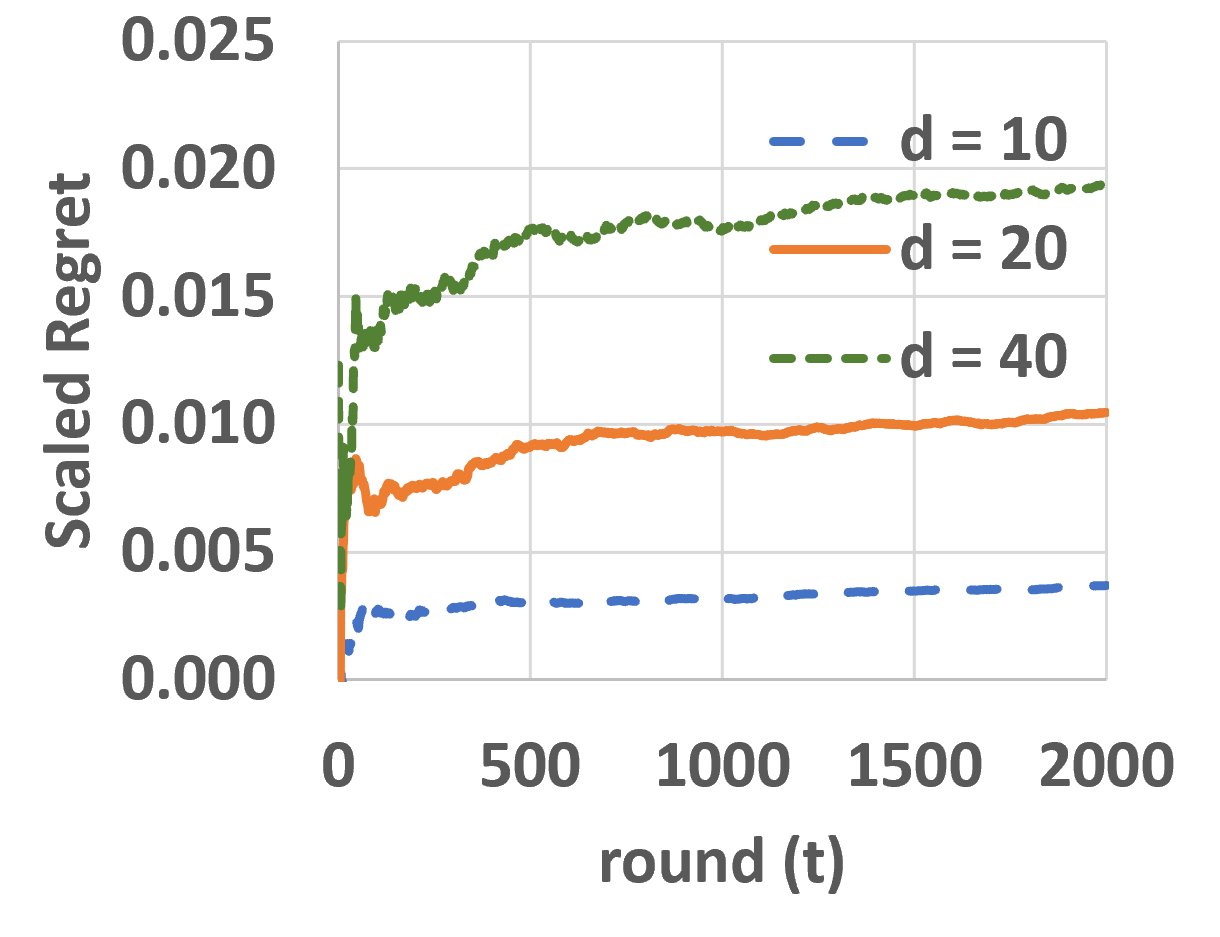}}
\subfigure[OGD of \citep{Flaxman+04} (squared $\loss$)]{\includegraphics[trim={0cm 0cm 0cm 0cm},clip,scale=0.45]{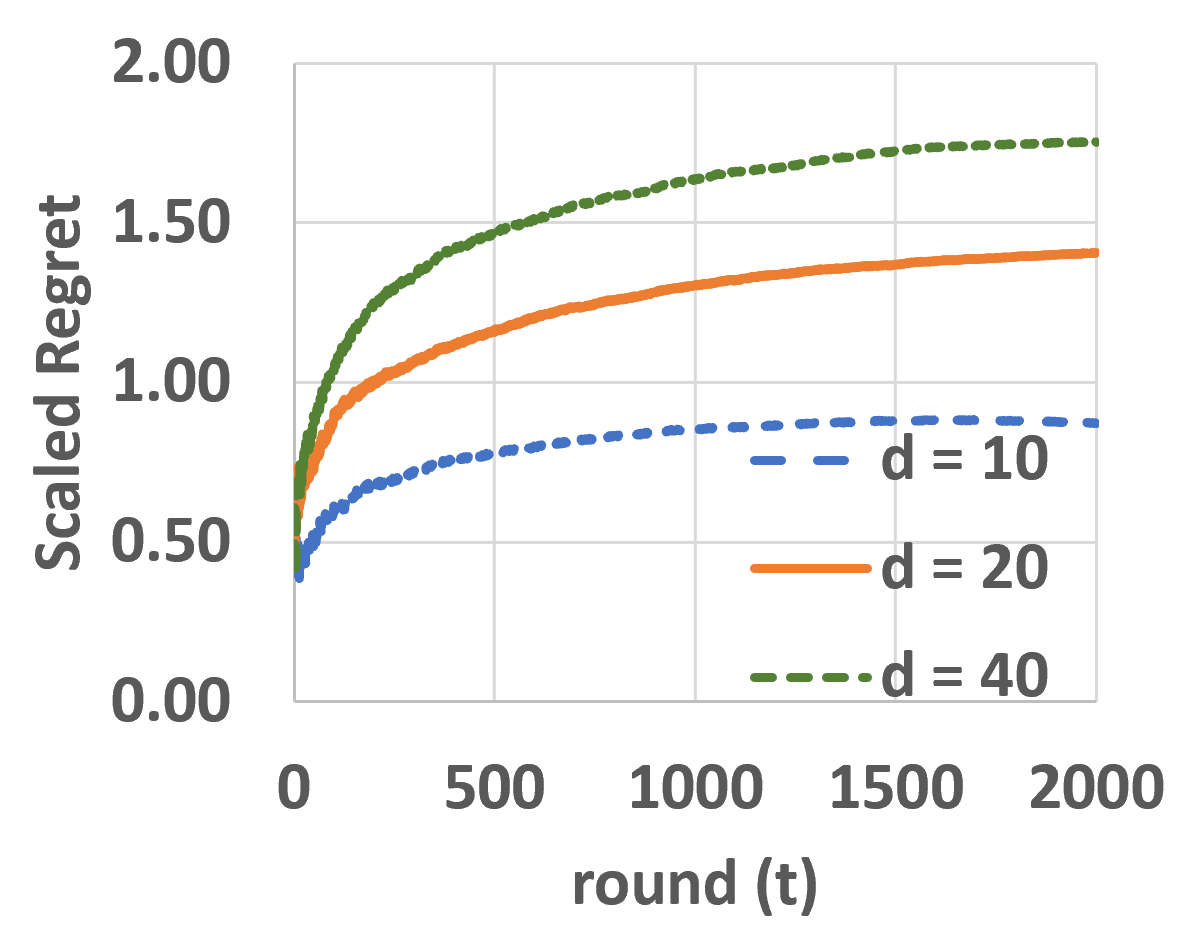}}
\subfigure[OGD of \citep{Flaxman+04} (abs. $\loss$)]{\includegraphics[trim={0cm 0cm 0cm 0cm},clip,scale=0.45]{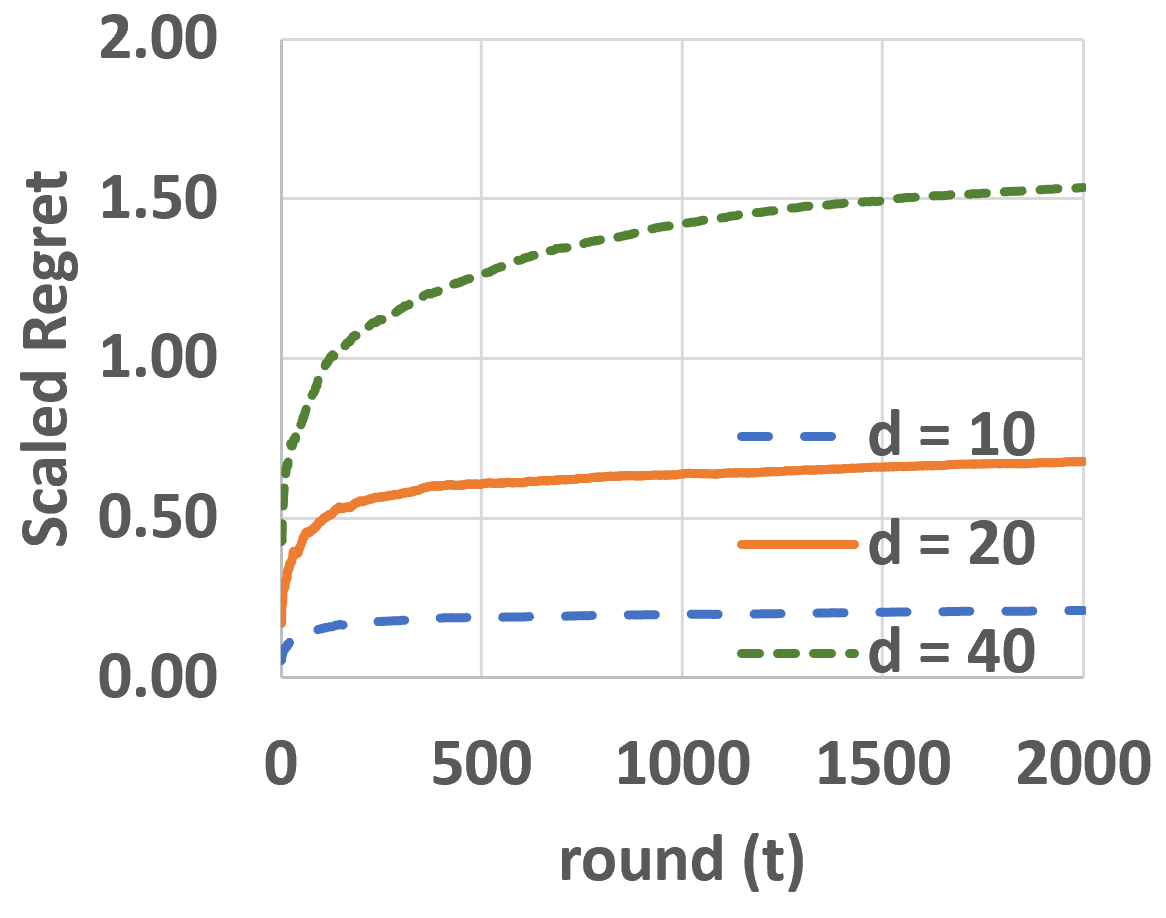}}
\caption{(a)-(b): Algorithm~\ref{algo:ogd}: Scaled cumulative regret $\cR_t/t^{3/4}$ vs. $t$ for the squared loss (a) and the absolute deviation loss (b). By Lemma~\ref{thm:ubogd}, the (scaled) regret converges to a numerical constant independent of $d$. (c)-(d): Algorithm~\ref{algo:kexp}: Scaled cumulative regret $\cR_t/\sqrt{t}$ vs. $t$ for the squared loss (c) and the absolute deviation loss (d). In accordance with Lemma~\ref{thm:ubkexp}, the (scaled) regret converges to a value proportional to $\sqrt{d}$. (e)-(f): OGD algorithm of~\citep{Flaxman+04}: Scaled cumulative regret $\cR_t/t^{3/4}$ vs. $t$ for the squared loss (e) and the absolute deviation loss (f). Compared to the corresponding pots in (a) and (b), it is evident that the regret is much higher here; in particular, in accordance with the result in~\citep{Flaxman+04}, the (scaled) regret converges to a value proportional to $\sqrt{d}$.}
\label{fig:results}
\end{figure*}

\paragraph{Comparison to~\citep{Flaxman+04}.} We present comparisons to the bandit OGD algorithm of~\citep{Flaxman+04} that does not exploit the pseudo-1d structure of the loss, achieving a regret of $O(\sqrt{d}T^{3/4})$, as against our Algorithm~\ref{algo:ogd} that achieves a regret of $O(T^{3/4})$.
In Figure~\ref{fig:results} (e)-(f), we show the expected regret of the bandit OGD algorithm of~\citep{Flaxman+04} on the same data as earlier (averaged over 50 problem instances), scaled by $1/t^{3/4}$ at round $t$, for the two loss functions; this, according to~\citep{Flaxman+04}, ensures that the regret converges to $O(\sqrt{d})$, with increasing rounds; notice that, e.g., in (e), for different $d$ values, we can infer that the ratio of the converged regrets of  ~\citep{Flaxman+04} and our algorithm (corresponding to plot (a)) is at most $3\sqrt{d}$; the additional constant factor also appears in the analysis of~\citep{Flaxman+04}.

\section{Conclusions and Future Work}
We have formulated a novel bandit convex optimization problem with pseudo-1d structure motivated by its applications in online decision making and large-scale parameter tuning in systems. We provide optimal minimax regret bounds for the pseudo-1d bandit convex optimization problem. An open question here is if there is a single algorithm that achieves the regret trade-off we show in the lower bound (as against our method, that relies on two schemes in two regimes of dimensionality of the problem). Another follow-up direction is to extend the results in this work to settings when $\pred_t$ is high-dimensional (when one needs to take multiple decisions based on the observed context), say $\pred_t(\mathbf{W};\x) = \mathbf{W}\x$, where the parameters to estimate are $\mathbf{W} \in \mathbb{R}^{m \times d}$.

\newpage
\bibliographystyle{plainnat}
\bibliography{pseudo-BCO-nuerips19}

\newpage
\appendix
\onecolumn{

\section*{\centering\Large{Supplementary: \SBCOexp}}

\section{Proofs}

\subsection{Proof of Theorem \ref{thm:lb} }

\begin{proof}
\textbf{Problem instance construction.} Divide the time interval $[T]$ into $d$ equal length sub intervals (hence each of length $\frac{T}{d}$) $T_1,\ldots, T_{d}$. Assume $T_0 = \emptyset$.

For $i \in [d]$: Choose $\sigma_i \sim \text{Ber}(\pm 1)$, and set $\x_i = \e_i$. Denote $\bsigma = (\sigma_1,\ldots,\sigma_d)$.

At any time $t \in T_i = \Big\{\frac{T}{d}(i-1)+1, \ldots, \frac{T}{d}i \Big \}$, $i \in [d]$,
\begin{enumerate}
\item Choose $\pred_t(\w;\x_t) = \w^\top\x_t$. Clearly $\nabla_{\w}(\pred_t(\cdot;\x_t)) = \x_t \in \{0,1\}^d$ which is revealed to the learner at the beginning of round $t$. We choose $\x_t = \x_i$.
\item Loss function $f_t(\w) = \loss_t(\w^\top\x_t) + \varepsilon_t = \mu\sigma_i(\w^\top\x_i) + \varepsilon_t, \text{ where } \varepsilon_t \sim \cN(0,\frac{1}{16}), \text{ for some constant } \mu > 0 \text{ (to be decided later)},\, \forall \w \in \cW$.
\item Learner plays $\w_t = [\w_t(1),\ldots,\w_t(d)] \in \cW$.

\end{enumerate}

Denote $\bw_i := \frac{1}{T_d}\sum_{t \in T_i}\w_t$, where $T_d = \frac{T}{d}$.

\begin{rem}[Optimum Point]
Note for any fixed $\w \in \cW$, the total expected loss is $\E\bigg[ \sum_{i = 1}^d\sum_{t \in T_i}f_t(\w) \bigg] = \frac{\mu T}{d}\sum_{i = 1}^d (\sigma_i\x_i^\top)\w = \frac{ T}{d}( \tilde \bsigma^\top\w)$, where $\tilde \sigma(i) = \mu \sigma_i,~ \forall i \in [d]$. Thus clearly the best point (i.e. the minimizer) $\w^* = -\frac{\bsigma}{\sqrt d}$. Note $\w^* \in \cW$.
\end{rem}



The expected regret of any $\cA$:

\begin{align}
\label{eq:eq1}
\nonumber \E[R_T] & = \sum_{i = 1}^d \sum_{t \in T_i} \mu[(\sigma_i\x_i^\top)\w_t - (\sigma_i\x_i^\top)\w^*] = \sum_{i = 1}^d \mu T_d \big[\E[\sigma_i\x_i^\top \bw_i] - (\sigma_i\x_i^\top)\w^*\big]\\
\nonumber & = \sum_{i = 1}^d T_d \E\bigg[ \mu\sigma(i)[\bw_i(i) - \w^*(i)] \bigg]\\
\nonumber & = \sum_{i = 1}^dT_d \E\bigg[ \mu\sqrt d \w^*(i)[\w^*(i) - \bw_i(i)] \bigg] = \sum_{i = 1}^dT_d \E\bigg[ \mu\sqrt d \Big((\w^*(i))^2 - \bw_i(i)\w^*(i)\Big) \bigg] \\
\nonumber & = \sum_{i = 1}^dT_d \E\bigg[ \mu\sqrt d \Big(\frac{1}{d} + \frac{\sigma_i}{\sqrt d}\bw_i(i)\Big) \bigg] \\
& = \sum_{i = 1}^dT_d \bigg[ \frac{2\mu}{\sqrt d} Pr(\sigma_i\bw_i(i) > 0) \bigg] \text{ since } \bw_i(i) \in \{-1/\sqrt{d},1/\sqrt{d}\}
\end{align}

Now for any $i \in [d]$:
\begin{align*}
Pr\big(\sigma_i\bw_i(i) > 0\big) &= \frac{1}{2}Pr\big(\bw_i(i)>0 \mid \sigma_i = +1\big) + \frac{1}{2}Pr\big(\bw_i(i)<0 \mid \sigma_i = -1\big)\\
& = \frac{1}{2}\Big(Pr\big(\bw_i(i)>0 \mid \sigma_i = +1\big) + 1 - Pr\big(\bw_i(i)>0 \mid \sigma_i = -1\big) \Big)\\
& \ge \frac{1}{2}\Big( 1 - | Pr\big(\bw_i(i)>0 \mid \sigma_i = +1\big) - Pr\big(\bw_i(i)>0 \mid \sigma_i = -1\big) | \Big),
\end{align*}

\begin{assump}
\label{assump1}
For proving the lower bound we assume that $\bar \w_i(i)$ is a deterministic function of the observed function values $\{f_t\}_{t \in T_i}$, respectively at $\{\w_t\}_{t \in T_i}$. Note that this assumption is without
loss of generality, since any random querying strategy can be seen as a randomization over deterministic querying strategies. Thus, a lower bound which holds uniformly for any
deterministic querying strategy would also hold over a randomization. Let us denote: $f([T_i]) = \{f_t\}_{t \in T_i}$.
\end{assump}

Then since the randomness of $\bar \w_i(i)$ only depends on $f([T_i])$, applying Pinsker's inequality, we get:

\begin{align*}
Pr \big(\sigma_i\bw_i(i) > 0\big) & \ge \frac{1}{2}\Big( 1 - \big| Pr\big(\sigma_i\bw_i(i)>0 \mid \sigma_i = +1\big) - Pr\big(\sigma_i\bw_i(i)<0 \mid \sigma_i = -1\big) \big| \Big)\\
& \ge \frac{1}{2}\bigg(1- \sqrt{2 KL\Big( P( f([T_i]) \mid \sigma_i=+1 ) || P(f([T_i]) \mid \sigma_i=-1)} \bigg)
\end{align*}

and further applying the chain rule of KL-divergence, we have:

\begin{align*}
Pr & \big(\sigma_i\bw_i(i) > 0\big) \ge \frac{1}{2}\bigg(1- \sqrt{2 \sum_{t \in T_i} KL\Big( P(f_t \mid \sigma_i=+1, \{f_\tau\}_{\tau \in [t-1] \setminus T_{i-1}}) || P(f_t \mid \sigma_i=-1, \{f_\tau\}_{\tau \in [t-1] \setminus T_{i-1}})} \bigg)\\
& \ge \frac{1}{2}\bigg(1-\sqrt{2\sum_{t \in T_i}\frac{4\mu^2\sigma_i^2\w_t(i)^2}{\frac{2}{16}}}\bigg) = \frac{1}{2}\bigg(1-\sqrt{\frac{64\mu^2 T_d}{d}}\bigg) \text{ since } \w_t(i)^2 = \frac{1}{d} \text{ and } \sigma_i^2 = 1
\end{align*}
where the last inequality follows by noting $P(f_t \mid \sigma_i, \{f_\tau\}_{\tau \in [t-1] \setminus T_{i-1} }) \sim \cN( \mu\sigma_i\w_t(i), \frac{1}{16})$, and \\ $KL(\cN(\mu_1,\sigma^2)||\cN(\mu_2,\sigma^2)) = \frac{(\mu_1-\mu_2)^2}{2\sigma^2}$ (for bounding the each individual KL-divergence terms).

\textbf{Case $1$ $(d \le 16\sqrt{T})$}

Combining the above claims with Eq. \eqref{eq:eq1}:
\begin{align*}
\E[R_T] & = \sum_{i = d}T_d \bigg[ \frac{2\mu}{\sqrt d} Pr(\sigma_i\bw_i(i) > 0) \bigg] \ge \sum_{i = d}T_d \bigg[ \frac{\mu}{\sqrt d} \big(1-8\mu\sqrt{\frac{T_d}{d}}\big)  \bigg],\\
& \ge \sum_{i = d}T_d \frac{1}{ 16 \sqrt T_d}\Big(1-\frac{1}{2}\Big) \text{ \bigg(setting } \mu  = \frac{\sqrt{d}}{16\sqrt T_d} \le 1\bigg) = \frac{\sqrt{d T}}{32}.
\end{align*}
Note that for any $t \in [T]$, $f_t$ s are $1$-lipschitz for $d \le 16\sqrt{T}$, as desired to understand the dependency of lower bound to the lipschitz constant.

\textbf{Case $2$ $(d > 16\sqrt{T})$}

In this case $T < \frac{d^2}{256}$.
Let us denote $d' = 16\sqrt{T} < d$, and let us use the above problem construction for dimension $d'$ (we can simply ignore decision coordinates $\w(d'+1), \ldots, \w(d)$, i.e. for any $\w \in \cW \subseteq \R^d$, denoting $\w_{[d']} = (\w_1,\ldots,\w_{d'})$, we can construct $f_t(\w) = f_t(\w_{[d']})$).

Now for the above problem suppose there exists an algorithm $\cA$ such that $\E[R_T(\cA)] \le \frac{\sqrt{d'T}}{32} = \frac{T^{3/4}}{32}$, then this violates the lower bound derived in \textbf{Case $1$}. Thus the lower bound for \textbf{Case $2$} is must be at least $\frac{T^{3/4}}{32}$.

Combining the lower bounds of \textbf{Case $1$ and $2$} concludes the proof.
\end{proof}


%

\subsection{Proof of Lemma~\ref{thm:ubkexp} and additional claims}

\textbf{Useful definitions and notation. } Before proceeding to the proof, we define relevant notation that will be used throughout this section. For the kernel $\K_t'$ (Definition \ref{def:kernel}), we define a linear operator $\K_t^{'*}$ on the space of functions $\predspace_t \mapsto \R$ as follows. For any function $\ell: \predspace_t \mapsto \R$:
\begin{equation}
	\label{eq:ktadj}
	\K_t^{'*}\ell(y) := \int_{y' \in \predspace_t} \ell(y') \K_t'(y',y) dy ~~~ \forall y \in \predspace_t,
\end{equation}
We also denote by $\cP$ and $\cQ_t$ the set of all probability measures on $\cW$ and $\cG_t$ respectively; and by $\bdelta_{y} \in \cQ_t$, $\bdelta_{\w} \in \cP$ the dirac mass at $y \in \cG_t$ and at $\w \in \cW$ respectively. For $\q \in \cQ_t$, define: 
\begin{equation*}
\big< \q, \ell \big> = \int_{y \in \predspace_t} \ell(y) \q (y) dy 
\end{equation*}
As noted in \cite{Bubeck17}, a useful observation on the operator \eqref{eq:ktadj} is that for any $\q \in \cQ_t$:
\begin{align}
\label{eq:ktadjprop}
\big< \K'_t\q,\loss_t \big > = \big< {\K'_t}^*\loss_t,\q \big>.
\end{align}

\textbf{Proof of Lemma \ref{thm:ubkexp}. }

\begin{proof}
For ease, we abbreviate $g_t(\w_t; \x_t)$ as $g_t(\w_t)$ throughout the proof.
We start by analyzing the expected regret w.r.t. the optimal point $\w^* \in \cW$ (denote $y_t^* = \pred_t(\w^*)$ for all $t \in [T]$). Define $\forall y \in \cG_t, ~\tilde \ell_t(y) := \tf_t(\w)$, for any $\w \in \cW(y)$.
Also let $\mathcal H_t = {\boldsymbol \sigma}\big(\{\x_\tau,\p_\tau,\w_\tau,f_\tau\}_{\tau = 1}^{t-1} \cup \{\x_t,\p_t\}\big)$ denote the sigma algebra generated by the history till time $t$. 
Then the expected cumulative regret of Algorithm \ref{algo:kexp} over $T$ time steps can be bounded as:
\begin{align}
\label{eq:reg_kp}
\nonumber &\E[R_T(\w^*)] := \E\bigg[\sum_{t = 1}^T\Big(f_t(\w_t) - f_t(\w^*)\Big)\bigg] = \E\bigg[\sum_{t = 1}^T\Big(\ell_t(g_t(\w_t)) - \ell_t(g_t(\w^*))\Big)\bigg]\\
\nonumber &= \E\bigg[\sum_{t = 1}^T\Big(\ell_t(y_t) - \ell_t(y_t^*)\Big)\bigg] = \E\bigg[\sum_{t=1}^{T}\big<\K'_t\q_t - \bdelta_{y^*_t}, \loss_t \big>\bigg] ~~[\text{since } y_t \sim K_t'\q_t]\\
\nonumber & \le \E\bigg[\sum_{t=1}^{T}\frac{3\epsilon L}{\lambda} + \frac{1}{\lambda} \big< \K'_t(\q_t - \bdelta_{y^*_t}), \loss_t \big>\bigg] ~~\big[ \text{from Property}\#2 \text{ of Lemma } \ref{lem:krnl_prop} \big]\\
\nonumber & \le 6\epsilon L T + 2\sum_{t=1}^{T}\E\bigg[\big< \K'_t(\q_t - \bdelta_{y^*_t}), \loss_t \big>\bigg] ~~\big[ \text{we can choose } \lambda = 1/2, \text{ see proof of Lemma } \ref{lem:krnl_prop}\big]\\
\nonumber & \overset{\text{by } \eqref{eq:ktadjprop}}= 6\epsilon L T  + 2\sum_{t=1}^{T}\E\bigg[\sum_{t=1}^{T}\big< {\K'}^*_t\loss_t,(\q_t - \bdelta_{y^*_t})\big>\bigg] \\
\nonumber & \overset{(a)}{=} 6\epsilon L T  + 2\sum_{t=1}^{T}\E\bigg[\sum_{t=1}^{T}\E_{y_t \sim \K'_t\q_t}\Big[\big< \q_t - \bdelta_{y^*_t}, \tilde \loss_t \big> \mid \mathcal H_t \Big]\bigg] \\
& = 6\epsilon L T + 2\sum_{t=1}^{T}\E\bigg[\sum_{t=1}^{T}\E_{y_t \sim \K'_t\q_t}\Big[\big< \p_t - \bdelta_{\w^*}, \tilde f_t \big> \mid \mathcal H_t  \Big]\bigg]
\end{align}
where the last equality follows by Lemma~\ref{lem:p_eqv}, and by $\big< \bdelta_{\w^*}, \tf_t \big> = \tf_t(\w^*) = \tilde \loss_t(y^*_t) = \big< \bdelta_{y^*_t}, \tilde \loss_t \big>$; the penultimate equality $(a)$ follows noting that for any $y' \in \cG_t$:
\[
\E_{y_t \sim \K'_t\q_t} [\tilde \loss_t(y')] = \int_{y_t \in \cG_t}\K'_t\q_t(y_t)\frac{\loss_t(y_t)}{\K'_t\q_t(y_t)}\K'_t(y_t,y')dy_t = \int_{y_t \in \cG_t}\loss_t(y_t)\K'_t(y_t,y')dy_t = {\K'_t}^*\loss_t(y').
\]
Let us denote by $\p^*$ a uniform measure on the set $\cW_\kappa:=\{\w \mid \w = (1-\kappa)\w^* + \kappa \w', \text{ for any } \w' \in \cW\}$ for some $\kappa \in (0,1)$. Note, this implies $\p^*(\w) = \begin{cases}
	\frac{1}{\kappa^d\text{vol}(\cW)}, ~~\text{ if } \w \in \cW_\kappa\\
	0 ~~\text{otherwise}
\end{cases}$.
 
Then note that: 
\begin{align*}
	\sum_{t=1}^{T}&\E_{y_t \sim \K'_t\q_t}\big< \p_t - \bdelta_{\w^*}, \tilde f_t \big>  = \sum_{t=1}^{T}\E_{y_t \sim \K'_t\q_t}[\big< \p_t ,\tilde f_t \big> - \big< \bdelta_{\w^*}, \tf_t \big>]\\
	& \overset{(a)}{=} \sum_{t=1}^{T}\E_{y_t \sim \K'_t\q_t}[\big< \p_t ,\tilde f_t \big>] - {\K'_t}^*\loss_t(g_t(\w^*))  \\
	& \overset{(b)}{\le} \sum_{t=1}^{T}\E_{y_t \sim \K'_t\q_t}[\big< \p_t ,\tilde f_t \big>] + \sum_{t = 1}^T\Big[ \kappa L D W - \big<{\p^*,\K'_t}^*\loss_t(g_t(\cdot))\big>\Big]  \\
	& = \sum_{t=1}^{T}\E_{y_t \sim \K'_t\q_t}[\big< \p_t ,\tilde f_t \big> - \big<\p^*,\tilde f_t\big>] + \kappa L D W T 
\end{align*}
where $(a)$ follows since $\E_{y_t \sim \K'_t\q_t} \big< \bdelta_{\w^*}, \tf_t \big> = \E_{y_t \sim \K'_t\q_t} [\tf_t(\w^*)] = \E_{y_t \sim \K'_t\q_t} [\tilde \loss_t(g_t(\w^*))] = {\K'_t}^*\loss_t(g_t(\w^*))$ as shown above; $(b)$ follows since by assumption $g_t$ is $D$ lipschitz and so by definition of $\cW_\kappa$ for any $\w \in \cW_\kappa$ we have $|g_t(\w) - g_t(\w^*)| \le D W$ (since $W = \text{Diam}(\cW)$). But from the Property \#$1$ of Lemma \ref{lem:krnl_prop} we have that the function ${\K'_t}^* \loss_t(\cdot)$ is $L$-lipschitz, which in turn implies for any $\w \in \cW_\kappa$, $|{\K'_t}^* \loss_t(g_t(\w)) - {\K'_t}^* \loss_t(g_t(\w^*))| \le L |g_t(\w) - g_t(\w^*)| \le \kappa L D W$. 
The last equality follows by applying the reverse logic used for $(a)$.

Combining above claims with \eqref{eq:reg_kp} we further get:
 \begin{align}
\label{eq:reg_kp2}
\E[R_T(\w^*)] \le 6\epsilon LT + 2\Bigg(  \kappa L D W T + \E\bigg[\sum_{t=1}^{T}\E_{y_t \sim \K'_t\q_t}\Big[\big< \p_t - \p^*, \tilde f_t \big> \mid \mathcal H_t  \Big]\bigg]\Bigg).
\end{align}

From Lemma \ref{lem:elem} we get:
\begin{align}
\label{eq:reg_kp5}
\sum_{t=1}^{T}\big< \p_t - \p^*, \tf_t \big> \le \dfrac{KL(\p^*||\p_1)}{\eta} + \frac{\eta}{2}\big< \p_t, \tf_t^2\big> = \dfrac{KL(\p^*||\p_1)}{\eta} + \frac{\eta}{2}\big< \q_t, \tilde \loss_t^2\big>,
\end{align}
where the equality $\big< \p_t, \tf_t^2\big> = \big< \q_t, \tilde \loss_t^2\big>$ follows from a similar derivation as shown in Lemma \ref{lem:p_eqv}. Now, note that:

\begin{align}
\label{eq:reg_kp3}
\nonumber \E_{y_t \sim \K'_t\q_t}\Big[\big< \q_t, \tilde \loss_t^2\big>\Big] & = \int_{y_t \in \cG_t}\hspace{-5pt}\K'_t\q_t(y_t)\big<\q_t, \tilde \loss_t^2\big>dy_t \\
\nonumber & = \int_{y_t \in \cG_t}\hspace{-5pt}\K'_t\q_t(y_t)\Big[ \int_{y \in \cG_t}\hspace{-5pt}\q_t(y)\frac{(\loss_t(y_t))^2}{(\K'_t\q_t(y_t))^2}(\K'_t(y_t,y))^2dy\Big]dy_t\\
& \le C^2\int_{y_t \in \cG_t}\dfrac{\K_t'^{(2)}\q_t(y_t)}{\K'_t\q_t(y_t)}dy_t \le BC^2,
\end{align}
where the last inequality follows from Property \#$3$ of Lemma \ref{lem:krnl_prop} with $B = 2\Big(1 + \ln \frac{1}{\epsilon}  + \ln\Big(\beta_\cW - \alpha_\cW\Big)\Big)$. 

Finally, by definition of $\p^*$, we can bound the KL divergence term as:
\begin{align}
\label{eq:reg_kp4}
KL(\p^*||\p_1) = d \log \frac{1}{\kappa}
\end{align}

Substituting \eqref{eq:reg_kp3} and \eqref{eq:reg_kp4} in \eqref{eq:reg_kp5}, letting $L' = L D W$, and setting $\kappa = \frac{1}{L'T}$, $\epsilon = \frac{1}{3L T}$,~\eqref{eq:reg_kp2} yields:
\begin{align*}
\E[R_T(\w^*)] & \leq 2 + 2\Bigg(1 + \dfrac{KL(\p^*||\p_1)}{\eta} + \frac{\eta}{2}\E\bigg[\sum_{t=1}^{T}\E_{y_t \sim \K'_t\q_t} \big< \q_t, \tilde \loss_t^2\big> \mid \mathcal H_t \bigg] \Bigg)\\
& =  4 + 2\Bigg(\frac{d\log L' T}{\eta}  + \frac{\eta B C^2T}{2} \Bigg)\\
& = 4 + 2\sqrt 2\Bigg(\sqrt{d B C^2T\log(L ' T)} \Bigg),
\end{align*}
where the last equality follows by choosing $\eta = \Big(\frac{2d\log(L' T)}{B C^2T}\Big)^{\frac{1}{2}}$. This concludes the proof. 
\end{proof}

\paragraph{Statements and proofs of additional lemmas used above:}

\begin{lem}
\label{lem:validp}
In Algorithm \ref{algo:kexp}, at any round $t$, both $\q_t \in \cQ_t$ and $\K_t'\q_t \in \cQ_t$. 
\end{lem}

\begin{proof}
Firstly note that, $\p_1 \in \cP$ simply by its initialization, and for any subsequent iteration $t=2,3,\ldots, T$, $\p_t \in \cP$ by its update rule.

Now for any $t \in [T]$ and $y \in \cG_t$, by definition $\q_t(y) > 0$, as $\p_t \in \cP$. The only remaining thing to prove is that $\int_{\cG_t}d\q_t(y) = 1$, which simply follows as:
\[
\int_{y \in \cG_t}\q_t(y)dy = \int_{y \in \cG_t}\int_{\cW_t(y)}\p_t(\w)d\w = \int_{\cW}\p_t(\w)d\w = 1 ~~[\text{since } \p_t \in \cP].
\]

Now, consider $\K_t'\q_t$. By definition, $\forall y \in \cG_t, \K_t'\q_t(y) = \int_{\predspace_t}\K_t'(y,y')d\q_t(y') > 0$ since by construction $\K_t'(y,\cdot) > 0$ and $\q_t 
\in \cQ_t$. Further, since $\int_{\predspace_t}\K_t'(y,y')d y = 1$ for every $y' \in \predspace_t$ (by construction), it is easy to show $\int_{\predspace_t}\K_t\q_t(y)dy = 1$ as follows:
\[
\int_{\predspace_t}\K_t'\q_t(y) d y = \int_{\predspace_t}\Big[ \int_{\predspace_t}\K_t'(y,y')d\q_t(y') \Big] d y = \int_{\predspace_t}\Big[ \int_{\predspace_t}\K_t'(y,y')d y \Big]d\q_t(y') = \int_{\predspace_t}d\q_t(y') = 1.
\]
\end{proof}

\begin{restatable}[]{lem}{kpeqv}
\label{lem:p_eqv}
At any round $t \in [T]$ of Algorithm \ref{algo:kexp},
$
\big<\p_t,\tf_t\big> = \big<\q_t,\tilde \loss_t\big>.
$
\end{restatable}

\begin{proof}
The claim follows from the straightforward analysis:
\begin{align*}
\big<\p_t,\tf_t\big> & = \int_{\w \in \cW}\p_t(\w)\tf_t(\w)d\w = \int_{y \in \cG_t}\int_{\w \in \cW_t(y)}\p_t(\w)\tf_t(\w)d\w\\
& = \int_{y \in \cG_t}\int_{\w \in \cW_t(y)}\p_t(\w)\tilde \loss_t(y)d\w
= \int_{y \in \cG_t}\tilde \loss_t(y)\int_{\w \in \cW_t(y)}\p_t(\w)d\w \\
& = \int_{y \in \cG_t}\tilde \loss_t(y)\q_t(y)dy = \big<\q_t,\tilde \loss_t\big>.
\end{align*}
\end{proof}

%
%


\begin{restatable}[]{lem}{elem}
\label{lem:elem}
Consider any sequence of functions $f_1, f_2, \ldots f_T$ such that $f_t: \cD \mapsto \R$ for all $t \in [T]$, $\cD \subset \R^d$  for some $d \in \N_+$. Suppose  $\cP$ denotes the set of probability measure over $\cD$. 
Then for any $\p \in \cP$, and given any $\p_1 \in \cP$, the sequence $\{\p_t\}_{t = 2}^{T}$ is defined as $\p_{t+1}(\w) := \dfrac{\p_t(\w)\exp\big( -\eta \f_t(\w) \big)}{\int_\tw \p_t(\tw)\exp\big( -\eta \f_t(\tw) \big)d\tw }$, for all $\w \in \cD$. Then it can be shown that:
\[
\sum_{t=1}^{T}\big< \p_t - \p, f_t \big> \le \dfrac{KL(\p||\p_1)}{\eta} + \frac{\eta}{2}\sum_{t=1}^{T}\big< \p_t, f_t^2\big>,
\]
where $KL(\p||\p_1)$ denotes the KL-divergence between the two probability distributions $\p$ and $\p_1$.
\end{restatable}

\begin{proof}
	We start by noting that by definition of KL-divergence:
	\[
	KL(\p||\p_t) - KL(\p||\p_{t+1}) = \int_{\cW}\p(\w)\ln \Big( \frac{\p_{t+1}(\w)}{\p_{t}(\w)}\Big) d\w.
	\]
Moreover, by definition of $\p_{t+1}$, $\frac{1}{\eta}\Big( KL(\p||\p_t) - KL(\p||\p_{t+1}) \Big) = \frac{1}{\eta}\bigg(\int_{\cW}\p(\w)\ln \Big( \frac{\p_{t+1}(\w)}{\p_{t}(\w)}\Big)\bigg) = -\E_{\p}[f_t(\w)] - \frac{1}{\eta}\ln \E_{\p_t}[e^{-\eta f_t(\w)}]$ for any $t = 1,2, \ldots, T$.
Then summing over $T$ rounds,
\[
\sum_{t=1}^{T} \Bigg[ -\E_{\p}[ f_t(\w)] - \frac{1}{\eta}\ln \E_{\p_t}[e^{-\eta f_t(\w)}] \Bigg] = \frac{1}{\eta}\Big( KL(\p||\p_1) - KL(\p||\p_{T+1})  \Big).
\]
Now adding $\sum_{t=1}^{T} f_t(\w_t)$ to both sides, this further gives:
\begin{align*}
 & \sum_{t=1}^{T}\Bigg[  f_t(\w_t) - \E_{\p}[ f_t(\w)]  \Bigg] = \frac{1}{\eta}\Big( KL(\p||\p_1) - KL(\p||\p_{T+1})  \Big) + \sum_{t=1}^{T}\Big( f_t(\w_t) + \frac{1}{\eta}\ln \E_{\p_t}[e^{-\eta f_t(\w)}]\Big)\\
& \hspace*{-0pt}\implies \sum_{t=1}^{T} \Bigg[  f_t(\w_t) - \E_{\p}[ f_t(\w)]  \Bigg] \le \frac{KL(\p||\p_1)}{\eta} + \sum_{t=1}^{T}\Big( f_t(\w_t) + \frac{1}{\eta}\ln \E_{\p_t}[e^{-\eta f_t(\w)}]\Big)\\
& \hspace*{-0pt}\implies \sum_{t=1}^{T} \E_{\w_t\sim\p_t}\Bigg[  f_t(\w_t) - \E_{\p}[ f_t(\w)]  \Bigg] \le \frac{KL(\p||\p_1)}{\eta} + \frac{1}{\eta}\sum_{t=1}^{T}\E_{\w_t\sim\p_t}\Bigg[\eta f_t(\w_t) + \ln \E_{\p_t}[e^{-\eta f_t(\w)}]\Bigg]\\
& \hspace*{-0pt}\implies \sum_{t=1}^{T} \Bigg[ \big<(\p_t-\p), f_t \big>  \Bigg] \le \frac{KL(\p||\p_1)}{\eta} + \frac{1}{\eta}\sum_{t=1}^{T}\E_{\w_t\sim\p_t}\Bigg [\eta  f_t(\w_t) +  \E_{\p_t}[e^{-\eta f_t(\w)}] - 1\Bigg] \\
& \hspace*{-10pt} \le  \frac{KL(\p||\p_1)}{\eta} + \frac{1}{\eta}\sum_{t=1}^{T}\E_{\w_t\sim\p_t} \Bigg[\eta f_t(\w_t)  + 1 -\eta \E_{\w \sim\p_t}[ f_t(\w)] + \E_{\w \sim\p_t}[ \frac{\eta^2f_t^2(\w)}{2}] - 1 \Bigg] \\
& \hspace*{-10pt} = \frac{KL(\p||\p_1)}{\eta} + \frac{\eta}{2}\sum_{t=1}^{T}\big<\p_t, f_t^2\big>,
\end{align*}
which concludes the proof. The last two inequalities above follow from $\ln s \le s - 1, ~ \forall s > 0$ and $e^{-s} \le 1-s+s^2/2, ~\forall s > 0$.
\end{proof}

\begin{restatable}[]{lem}{lemkrnl}
\label{lem:krnl_prop}
For any convex and $L$-Lipschitz function, $\loss:\cG_t \mapsto \R_+$, such that $\cG_t = [\alpha,\beta] \subseteq \R$, $\q \in \cQ_t$, and any $y \in \cG_t$, the kernel $\K'_t: \cG_t\times\cG_t \mapsto \R_+$ satisfies:
\begin{enumerate}
    \item The function ${\K'_t}^*\loss (\cdot)$ is $L$-Lipschitz. 
	\item ${\K'_t}^*\loss(y) \le (1-\lambda)\big<\K'_t\q,\loss\big> + \lambda \loss(y) + 3 \epsilon L$, where $\lambda$ is a constant.
	\item For any $\q \in \mathcal \cQ_t$, define operator ${\K_t'}^{(2)}\q: \cG_t \mapsto \R$ as:
\[
{\K_t'}^{(2)}\q(y) := \int_{y' \in \predspace_t}(\K_t'(y,y'))^2d\q(y') ~~~ \forall y \in \predspace_t,
\] 
then $\displaystyle \int_{y \in \cG_t}\dfrac{{\K'_t}^{(2)}\q(y)}{\K'_t\q y}dy \le B$, where $B =  2\bigg(1 + \ln \frac{1}{\epsilon}  + \ln\Big(\beta - \alpha\Big)\bigg)$.
\end{enumerate}
\end{restatable}

\begin{proof}

$1.$ For the first part, let us denote $\bar y = \E_{y \sim \q}[y]$. Then note that:

\begin{align}
\label{eq:krnl_prop1}
{\K'_t}^*\loss(y) = \big< \K'_t\delta_y,\loss \big> =
\begin{cases}
\E_{U \sim \text{unif}[0,1]}\big[ \loss( U \bar y + (1-U)y) \big], ~~~ \text{ if } |y - \bar y| \ge \epsilon\\
\E_{U \sim \text{unif}[0,1]}\big[ \loss( \bar y - \epsilon U) \big], ~~~ \text{ if } |y - \bar y| < \epsilon
\end{cases},
\end{align}
which immediately implies the function ${\K'_t}^*\loss(\cdot)$ has the same Lipschitz parameter that of $\loss(\cdot)$.

$2.$ We prove this part considering two cases separately:

\textbf{Case 1.} $|y - \bar y| \ge \epsilon$:
By construction of $\K_t'$ (see Definition \ref{def:kernel}), we note that expectation of $y$ w.r.t. $\q$ and $\K_t'\q$, i.e. respectively $\bar y = \E_{y \sim \q}[y]$ and $\E_{y \sim \K_t'\q}[y]$ can differ at most by $2 \epsilon$, i.e. $|\E_{y \sim \q}[y] - \E_{y \sim \K_t'\q}[y]| \le 2 \epsilon$ \citep{Bubeck17}. We write, $\E_{y \sim \K_t'\q}[y] = \E_{y \sim \q}[y] + \psi$, clearly $\psi \in [-2\epsilon, 2\epsilon]$.  Hence:
\begin{align}
	\label{eq:temp1}
	\nonumber \loss(\bar y) & = \loss(\E_{y \sim \K_t'\q}[y] - \psi) \\
	\nonumber & \le \loss\bigg(\int_{y \in \cG_t}y \K'_t\q(y)dy \bigg) + \psi L \le \int_{y \in \cG_t}\loss(y) \K'_t\q(y)dy + 2\epsilon L \\
	& = \big< \K'_t\q,\loss \big> + 2\epsilon L
\end{align}
where the first inequality follows using the $L$-lipschitzness of $\loss$ and the second inequality follows using Jensen's inequality (since $\loss$ is convex). Now consider the case $|y - \bar y| \ge \epsilon$ in~\eqref{eq:krnl_prop1}:
\begin{align*}
{\K'_t}^*\loss(y) & = \E_{U \sim \text{unif}[0,1]}\big[ \loss( U \bar y + (1-U)y) \big] \le \frac{\loss(\bar y) + \loss(y)}{2} \\ 
& \overset{\text{ by } \eqref{eq:temp1}}{\le} \frac{\big<\K'_t\q,\loss\big> + \loss(y)}{2} + \epsilon L
\end{align*}
This shows that for this case the claim of Part $(2)$ holds for $\lambda = \frac{1}{2}$. \\
\textbf{Case 2.} $|y - \bar y| < \epsilon$: 
Note $\bar y - \epsilon U \in [\bar y - \epsilon, \bar y]$ in \eqref{eq:krnl_prop1}. And in this case $\loss(\bar y) \leq \loss(y) + \epsilon L$. Using the fact that $\loss(\cdot)$ is convex and $L$-lipschitz, by similar arguments used to obtain~\eqref{eq:temp1} above, we have: 
\begin{align*}
{\K'_t}^*\loss(y) \le \loss(\bar y) + \epsilon L = \loss(\bar y)/2 + \loss(\bar y)/2 + \epsilon L \leq \langle\K'_t\q, \loss\rangle/2 + (\loss(y)+\epsilon L)/2 + 2\epsilon L  
\end{align*}
which implies for this case as well, the claim of Part $(2)$ holds for $\lambda = 1/2$.

3. For this part, note that:
\begin{align*}
\int_{y \in \cG_t}&\frac{{\K'_t}^{(2)}\q( y)}{\K'_t\q( y)}d y \overset{(a)}{\le} \int_{\alpha}^{\beta} \frac{1}{\max(|y - \bar y|,\epsilon)}dy \\
& = \int_{\alpha}^{\bar y - \epsilon} \frac{1}{\max(|y - \bar y|,\epsilon)}dy + \int_{\bar y - \epsilon}^{\bar y + \epsilon} \frac{1}{\max(|y - \bar y|,\epsilon)}dy + \int_{\bar y + \epsilon}^{\beta} \frac{1}{\max(|y - \bar y|,\epsilon)}dy\\
& = \int_{\alpha}^{\bar y - \epsilon} \frac{1}{\bar y-y}dy + \int_{\bar y - \epsilon}^{\bar y + \epsilon} \frac{1}{\epsilon}dy + \int_{\bar y + \epsilon}^{\beta} \frac{1}{y - \bar y}dy\\
& = \frac{1}{\epsilon}\int_{\bar y - \epsilon}^{\bar y + \epsilon}dy + 2\ln \frac{1}{\epsilon} + \ln(\beta - \bar y) + \ln(\bar y - \alpha)\\
& \le 2\bigg(1 + \ln \frac{1}{\epsilon}  + \ln\Big(\beta - \alpha\Big)\bigg) ~~ (\text{since } \alpha \le \bar y \le \beta)
\end{align*}
where $(a)$ follows noting $\K'_t(y,y') \le  \frac{1}{\max(|y - \bar y|,\epsilon)}, \, \forall y,y' \in \cG_t$ which implies ${\K'_t}^{(2)}\q(y) \le  \frac{{\K'_t}\q(y)}{\max(|y - \bar y|,\epsilon)}$.
\end{proof}

\subsection{Proof of Lemma \ref{lem:gradest} }

\begin{proof}
For any $\loss_t: \R \to [0,C]$, $t \in [T]$, define $\hat \loss_t: \R \mapsto [0,C]$ such that $\hat \loss_t(y) = \E_{u \sim \bU\big(\cB_1(1)\big)}\loss_t(y + \delta u)$, for any $y \in \R$. Let us also define $\hf_t(\w) = \hl_t(\pred_t(\w;\x_t)), \forall \w \in \cW$. Let $y_t = \pred_t(\w_t;\x_t)$, $\forall t \in [T]$. 

Then given any fixed $\w \in \cW$ and $\x \in \R^d$, by chain rule $\nabla_{\w} \hat f_t(\w) = \frac{d \hat f_t(y)}{dy}\nabla_{\w} (\pred_t(\w;\x_t)) = \dfrac{d \hat \loss_t(y)}{dy}\nabla_{\w}(\pred_t(\w_t;\x_t))$. Consider the RHS of the lemma equality:

\begin{align*}
&\E_{u \sim \bU(\cS_1(1))}\Big[\frac{1}{\delta}{\loss_t\big(\pred_t(\w_t;\x_t) + \delta u\big)}u \mid \w_t \Big]\nabla_{\w}(\pred_t(\w_t;\x_t)) \\
& = \frac{d \hat \loss_t(y_t)}{dy_t}\nabla_{\w}(\pred_t(\w_t;\x_t)) = \nabla_{\w}\hat \f_t(\w_t) = \nabla_\w  \E_{u}\big[ \loss_t(\pred_t(\w_t; \x_t) + \delta u) \big],
\end{align*}
where the first equality is due to Lemma $1$ of \cite{Flaxman+04} applied to the 1-dimensional ball $\cB_1(1)$. 
\end{proof}

\subsection{Proof of Lemma \ref{thm:ubogd} }

\begin{proof}
We start by recalling Lemma $2$ of \cite{Flaxman+04} that   uses the online gradient descent analysis by \cite{Zink03} with unbiased random gradient estimates.
We restate the result below for convenience:

\begin{lem}[Lemma $2$, \cite{Flaxman+04}]
\label{lem:kalai}
Let $S \subset \cB_d(R) \subset \R^d$ be a convex set, $f_1, f_2, \ldots, f_T : S \mapsto \R$ be a sequence of convex, differentiable functions. Let $\w_1, \w_2, \ldots, \w_T \in S$ be a sequence of predictions defined as $\w_1 = 0$ and $\w_{t+1} = \emph{\bP}_S(\w_t - \eta h_t)$, where $\eta > 0$, and $h_1, h_2, \ldots, h_T$ are  random variables such that $\E[h_t \big | \w_t] = \nabla f_t(\w_t)$, and $\|h_t\|_2 \le G$, for some $G > 0$ then, for $\eta = \frac{R}{G\sqrt T}$ the expected regret incurred by above prediction sequence is:
\[
\E\bigg[ \sum_{t = 1}^T f_t(\w_t) \bigg] - \min_{\w \in S}\sum_{t = 1}^T f_t(\w) \le RG\sqrt T.
\]
\end{lem}

Coming back to our problem setup, let us first denote $\hf_t(\w) = \hl_t(\pred_t(\w;\x_t))$, for all $\w \in \cW, \, t \in [T]$ (recall from the proof of Lemma \ref{lem:gradest}, we define $\hat \loss_t: \R \mapsto [0,C]$ such that $\hat \loss_t(y) = \E_{u \sim \bU(\cB_1(1))}\loss_t(y + \delta u)$, for any $y \in \R$).
We can now apply Lemma \ref{lem:kalai} in the setting of Algorithm \ref{algo:ogd} on the sequence of convex (by \textbf{(A1) (ii)}), differentiable functions $\hf_1, \hf_2, \ldots \hf_T: \cW_\alpha \mapsto [0,C]$, with $h_t  = \frac{1}{\delta}\big({\loss_t(a_t)}u\big)\nabla\pred_t(\w_t;\x_t)$, with $u \sim \cB_1(1)$ (note that Lemma \ref{lem:gradest} implies $\E[h_t\big | \w_t] = \nabla_{\w} \hf_t(\w_t) = \nabla_\w  \E_{u}\big[ \loss_t(\pred_t(\w_t; \x_t) + \delta u) \big]$). We get:

\begin{equation}
\label{eq:prf_algadv1}
\E\bigg[ \sum_{t = 1}^T \hf_t(\w_t) \bigg] - \min_{\w \in \cW_\alpha}\sum_{t = 1}^T \hf_t(\w) \le \frac{WDC \sqrt T}{\delta},
\end{equation}

as in this case $R \le (1-\alpha)W < W$, and, by \textbf{(A3) (ii)}, $\|h_t\| = \|\frac{1}{\delta}\big({\loss_t(a_t)}u\big)\nabla(\pred_t(\w_t;\x_t))\| \le \frac{DC}{\delta}$, so $G = \frac{DC}{\delta}$, and $\eta = \frac{W\delta}{DC\sqrt{T}}$.
Further, since $\loss_t(\cdot)$s are assumed to be $L$-Lipschitz, \eqref{eq:prf_algadv1} yields:

\begin{align*}
& \E\bigg[ \sum_{t = 1}^T  (f_t(\w_t) - \delta L) \bigg] - \min_{\w \in \cW_\alpha}\sum_{t = 1}^T  (f_t(\w) + \delta L)  \le \frac{WDC \sqrt T}{\delta},\\
\implies & \E\bigg[ \sum_{t = 1}^T  f_t(\w_t) \bigg] - \min_{\w \in \cW_\alpha}\sum_{t = 1}^T  f_t(\w)  \le \frac{WDC \sqrt T}{\delta} + 2 \delta LT \\
\implies & \E\bigg[ \sum_{t = 1}^T  f_t(\w_t) \bigg] - \min_{\w \in \cW}\sum_{t = 1}^T  f_t(\w)  \le \frac{WDC \sqrt T}{\delta} + 2 \delta LT + \alpha LT,\\
\implies & \E\bigg[ \sum_{t = 1}^T  f_t(\w_t) \bigg] - \min_{\w \in \cW}\sum_{t = 1}^T  f_t(\w)  \le \frac{WDC \sqrt T}{\delta} + 3 \delta LT,
\end{align*}
setting $\alpha = \delta$.
The claim follows minimizing the RHS above w.r.t. $\delta$. Setting $\delta = \Big( \frac{WDC}{3L \sqrt T} \Big)^{1/2}$ gives:
\begin{align*}
\E[\regret_T(\learner)] = \E\bigg[ \sum_{t = 1}^T  f_t(\w_t) \bigg] - \min_{\w \in \cW}\sum_{t = 1}^T  f_t(\w)  \le 2\sqrt{3WLDC} T^{3/4},
\end{align*}
which concludes the proof.
\end{proof}

\section{Appendix for Simulations (Section~\ref{sec:exp})}

\paragraph{Implementation details of Algorithm~\ref{algo:kexp}.} The main challenge in implementing \EXPBCO(Algorithm~\ref{algo:kexp}) is to handle the continuous `action space' $\cW$; in particular, to maintain and update the probability distribution $\p_t$ over $\cW$, and to sample from $\p_t$ given $y_t$ at round $t$. Towards this we use an epsilon-net trick to discretize $\cW$ into finitely many points---specifically, since we choose $\cW = \cB_d(1)$, we discretize the $[0,1]$ interval every $d$ direction with a grid size of $O(\nicefrac{1}{d})$, and consider only the points inside $\cB_d(1)$. This reduces the action space $\cW$ into finitely many points (say $N$), and we now proceed by maintaining and updating probabilities on every such discrete point following  the steps of Algorithm~\ref{algo:kexp} (we initialize $\p_1 \leftarrow \nicefrac{1}{N}$ for all $N$ points in the epsilon net).


\paragraph{Implementation details for Algorithm~\ref{algo:kep}}

}


\end{document}